\documentclass[11pt]{article}
\usepackage{format}
\usepackage{macros}

\usepackage{subcaption}
\usepackage{graphicx}

\newcommand{\IH}{I^{\textrm{ent}}}
\newcommand{\LH}{L^{\textrm{ent}}}
\newcommand{\Imp}{\mathrm{Imp}}

\title{Tracking and Improving Information in the Service of Fairness}
\author{Sumegha Garg\thanks{\texttt{sumeghag@cs.princeton.edu}.  Part of this work completed while visiting Stanford University.}\\Princeton University \and Michael P.~Kim\thanks{\texttt{mpk@cs.stanford.edu}. Part of this work completed while visiting the Weizmann Institute of Science.  Supported, in part, by a Google Faculty Research Award, CISPA Center for Information Security, and the Stanford Data Science Initiative.}\\Stanford University \and Omer Reingold\thanks{\texttt{reingold@stanford.edu}.  Supported by NSF Grant CCF-1763311.}\\Stanford University}
\date{}

\begin{document}
\maketitle

\begin{abstract}
As algorithmic prediction systems have become widespread, fears that these systems may inadvertently discriminate against members of underrepresented populations have grown.
With the goal of understanding fundamental principles that underpin the growing number of approaches to mitigating algorithmic discrimination, we investigate the role of \emph{information} in fair prediction.
A common strategy for decision-making uses a \emph{predictor} to assign individuals a risk score; then, individuals are selected or rejected on the basis of this score.
In this work, we study a formal framework for measuring the \emph{information content} of predictors.
Central to the framework
is the notion of a \emph{refinement}, first studied by \cite{degroot}.
Intuitively, a refinement of a predictor $z$ increases the overall informativeness of the predictions without losing the information already contained in $z$. We show that increasing information content through refinements improves the downstream selection rules across a wide range of fairness measures (e.g.\ true positive rates, false positive rates, selection rates).
In turn, refinements provide a simple but effective tool for reducing disparity in treatment and impact without sacrificing the utility of the predictions.
Our results suggest that in many applications, the perceived ``cost of fairness'' results from an information disparity across populations, and thus, may be avoided with improved information.
\end{abstract}

\section{Introduction}
\label{sec:intro}

As algorithmic predictions
are increasingly employed as parts of systems that
\emph{classify people}, concerns that such classfiers
may be \emph{biased} or \emph{discriminatory}
have increased correspondingly.  These concerns are
far from hypothetical; disparate treatment on the basis
of \emph{sensitive features}, like race and gender,
has been well-documented in diverse algorithmic application
domains \cite{gendershades, wordvecs, korolovaFB}.
As such, researchers across fields
like computer science, machine learning, and economics have
responded with many works aiming to address the serious issues
of fairness and unfairness that arise in automated decision-making
systems.\footnote{Moritz Hardt's
lecture communicates this trend quite succinctly.
\url{https://fairmlclass.github.io/1.html#/4}}

While most researchers studying algorithmic fairness can
agree on the high-level objectives of the field (e.g.,\ \emph{to ensure individuals are not mistreated on the basis
of protected attributes; to promote social well-being and justice
across populations}), there
is much debate about how to translate these normative aspirations
into a concrete, formal \emph{definition} of what it means for
a prediction system to be fair.
Indeed, as this nascent field has progressed, the
efforts to promote ``fair prediction'' have grown increasingly
divided, rather than coordinated.  Exacerbating the problem,
\cite{mitchell} identifies that each new approach to fairness
makes its own set of assumptions, \emph{often implicitly},
leading to contradictory notions about the right way to
approach fairness \cite{chouldechova,kmr}; 
these inconsistencies add to the ``serious challenge for
cataloguing and comparing defintions'' \cite{mitchell}.
Complicating matters further, recent works \cite{delayed,goel}
have identified shortcomings of many well-established notions of
fairness.  At the extreme, these works argue that
blindly requiring certain statistical fairness conditions
may in fact \emph{harm} the communities they
are meant to protect.

The state of the literature makes clear that choosing
an appropriate notion of fairness for prediction tasks
is a challenging affair.
Increasingly, fairness is viewed as a context-dependent
notion \cite{selbst,fifty}, where the ``right'' notion for a given task
should be informed by conversations between
computational and social scientists.
In the hopes of unifying some of
the many directions of research in the area,
we take a step back and ask whether
there are guiding princples
that broadly serve the high-level goals of ``fair'' prediction,
without relying too strongly on any specific notion of fairness.
The present work argues that understanding the
``informativeness'' of predictions
needs to be part of any sociotechnical conversation about the
fairness of a prediction system.
Our main contribution is to provide a technical
language with strong theoretical backing to
discuss informational issues in the context of fair prediction.

\vspace{-11pt}
\paragraph{Our contributions.}
Towards the goal of understanding common themes
across algorithmic fairness,
we investigate the role of \emph{information} in fair prediction.
We study a formal notion of informativeness in predictions
and demonstrate that it serves as an effective tool
for understanding and improving
the utility, fairness, and impact of downstream decisions.
In short, we identify that many ``failures'' of
requiring fairness in prediction systems can be
explained by an \emph{information disparity}
across subpopulations.
Further, we provide algorithmic tools that aim to
counteract these failures by improving informativeness.
Importantly, the framework is not wedded to any specific fairness
desideratum and can be applied broadly to prediction
settings where discrimination may be a concern.
Our main contributions can be summarized as follows:
\vspace{-7pt}
\begin{itemize}
\item First and foremost, we identify informativeness
as a key fairness desideratum.
We provide information-theoretic and algorithmic tools for reasoning
about how much an individual's prediction
reveals about their eventual outcome.
Our formulation clarifies the intuition that more
informative predictions should enable fairer
outcomes, across a wide array of interpretations
of what it means to be ``fair.''
Notably,
\emph{calibration} plays a key technical role in this
reasoning; the information-theoretic framework we present
relies intimately on the assumption that the underlying
predictors are calibrated.
Indeed, our results
demonstrate a surprising application of these calibration-based methods
towards improving parity-based fairness criteria,
running counter to the conventional wisdom that
calibration and parity are completely at odds with one another.
\item In Section~\ref{sec:info}, we provide a self-contained exposition of the framework we use to study
the \emph{information content} of a
predictor.  Information content formally quantifies
the uncertainty over individuals' outcomes given their
predictions.
Leveraging properties of calibrated predictors, we show that the information content of a predictor is directly related to the \emph{information loss}
between the \emph{true} risk distribution and the \emph{predicted} risk distribution.
Therefore, in many cases, information content -- a measurable characteristic of the  
predicted risk distribution -- can serve as a proxy for reasoning about the information disparity across groups.
To compare the information content
of multiple predictors, we
need a key concept called a \emph{refinement};
informally, a refinement
of a predictor increases the overall information content,
without losing any of the original information.
Refinements provide the technical tool for reasoning
about how to \emph{improve} prediction quality.
\item In Section~\ref{sec:value}, we revisit the question
of finding an optimal fair selection rule.
For prominent parity-based
fairness desiderata, we show that the optimal selection
rule can be characterized as the solution to a certain
linear program, based on the given predictor.
We prove that improving the
information content of these predictions via refinements results in a Pareto improvement
of the resulting program in terms of
utility, disparity, and long-term impact.
As one concrete example, if we hold the selection rule's
utility constant, then \emph{refining the underlying
predictions causes the disparity between groups to decrease}.
Additionally, we prove that at times, the \emph{cost} associated
with requiring fairness should be blamed on a
\emph{lack of information} about important subpopulations,
not on the fairness desideratum itself.
\item In Section~\ref{sec:blm}, we describe a simple algorithm,
\texttt{merge},
for incorporating disparate sources of information into
a single calibrated predictor.  The
\texttt{merge} operation can be implemented efficiently,
both in terms of time and sample complexity.
Along the way in our analysis, we introduce the concept
of \emph{refinement distance} -- a measure of how much
two predictors' knowledge ``overlaps''
-- that may be of independent interest.
\end{itemize}

Finally,
a high-level contribution of the present work is to
\emph{clarify challenges} 
in achieving fairness in prediction tasks.  Our framework for tracking
and improving information is particularly compelling because
it does not requires significant technical background
in information theory nor algorithms to understand.
We hope the framework will facilitate interactions between
computational and social scientists to further
unify the literature on fair prediction, and ultimately,
effect change in the fairness of real-world prediction systems.

\subsection{Why information?}

We motivate the study of information content in
fair prediction by giving an intuitive overview of
how information disparities can lead to unfair treatment
and impact.
We scaffold our discussion around examples from
two recent works \cite{delayed,goel} that raise concerns
about using broad-strokes statistical tests as the
\emph{definition} of fairness.
This overview will be informal, prioritizing
intuition over technicality; see Section~\ref{sec:prelim}
for the formal preliminaries.

We consider a standard prediction setting
where a decision maker, who we call the \emph{lender},
has access to a \emph{predictor} $z:\X \to [0,1]$;
from the predicted risk score $z(x)$,
the decison maker must choose whether to
accept or reject the individual $x \in \X$,
i.e.\ whether to give $x$ a loan or not.  For each individual
$x \in \X$, we assume there is an associated outcome
$y \in \set{0,1}$ representing if they would default
or repay a given loan ($0$ and $1$, respectively).
Throughout, we will be focused on \emph{calibrated}
predictors.  Intuitively, calibration requires
that a predicted score of $z(x) = p$ corresponds to the same
level of risk, regardless of whether  $x \in A$ or $x \in B$.
More technically, this means that we can think of $z(x) = p$
as a conditional \emph{probability}; that is, amongst the individuals
who receive score $z(x) = p$, a $p$-fraction of them end up
having $y=1$.
For simplicity, we assume there are two disjoint subpopulations
$A,B \subseteq \X$.  The works of \cite{delayed,goel} mainly focus
on settings where there are material differences between
the distribution of $y$ in
the populations $A$ and $B$, arguing that these
differences can lead to undesirable outcomes.
We argue that even if the true risk of individuals
from $A$ and $B$ are identically distributed,
differences in the distribution of \emph{predicted} risk scores
give rise to the same pitfalls.

\vspace{-11pt}
\paragraph{A caution against parity.}
\cite{delayed} focuses on notions of fairness that require
parity between groups.  One notion they
study is \emph{demographic parity}, which requires that
the selection rate between groups $A$ and $B$ be
equal; that is, $\Pr[x\text{ selected} \given x \in A]
= \Pr[x\text{ selected} \given x \in B]$.  Suppose
that the majority of applicants come from group $A$
and that on-average, members of $A$ tend to have higher
predictions according to $z$.  In such a setting, an
unconstrained utility-maximizing lender would give out
loans at a higher rate in $A$ than in $B$.
The argument in \cite{delayed} against requiring
demographic parity goes as follows:
a lender who is constrained to satisfy demographic parity
must either give out fewer loans in $A$ or more in $B$;
because the lender does not want to give up utility from
loaning to $A$, the constrained lender will give
out more loans in $B$. \cite{delayed} argue that
in many reasonable settings, the lender will end up loaning
to \emph{underqualified} individuals in $B$ who are unlikely to repay;
thus, the default rate in $B$ will increase significantly.
In their model, this increased default rate translates into
\emph{negative impact} on the population $B$, whose members
may go into debt and become even less creditworthy.

\vspace{-11pt}
\paragraph{A caution against calibration.}
In general, \cite{goel} advocates for the use of calibrated
score functions paired with threshold selection policies, where
an individual is selected if $z(x) > \tau$ for some fixed,
group-independent threshold $\tau$.  Still, they caution
that threshold policies paired with calibrated predictors
are not sufficient to guarantee equitable treatment.
In particular, suppose that
the lender is willing to accept individuals if they
have at least $0.7$ probability of returning the loan.
But now consider a set of calibrated risk scores where the scores are
much more confident about $A$ than about
$B$; at the extreme, suppose that
for $x \in A$, $z(x) \in \set{0,1}$ (i.e.\ perfect predictions) and for $x \in B$, $z(x)=0.5$ (i.e.\ uniform predictions).
In this case, using a fixed threshold of $\tau = 0.7$
will select every qualified individual in $A$ and none of the
individuals from $B$, even though, by the fact that $z$ is calibrated, half of them were qualified.
Worse yet, even if we try to select more members of $B$,
every member of $B$ has a $0.5$ probability of defaulting.
Indeed, in this example, we cannot distinguish between the
individuals in $B$ because they all receive the same score $z(x)$.
In other words, we have no \emph{information} within
the population $B$ even though the predictor was calibrated.

These examples make clear that when there are actual
differences in the risk score distributions between
populations $A$ and $B$, seemingly-natural approaches
to ensuring fairness -- enforcing parity amongst groups
or setting a group-independent threshold -- may result
in a disservice to the underrepresented population.
These works echo a perspective raised by \cite{fta}
that emphasizes the distinction between requiring broad-strokes demographic parity as a \emph{constraint} versus stating parity as a \emph{desideratum}.
Even if we believe that groups
should ideally be treated similarly, defining fairness as
satisfying a set of hard constraints may have unintended
consequences.

Note that the arguments above relied on
differences in the predicted risk scores $z(x)$ for
$x \in A$ and $x \in B$, but not the true underlying
risk.  This observation has two immediate corollaries.
On the one hand, in both of these vignettes,
if the \emph{predicted} score distributions
are different between population $A$ and $B$, then
such approaches to fairness could still cause harm,
\emph{even if the true score distributions are
identically distributed}.  On the other hand,
just because $A$ and $B$ look different according to
the predicted scores, they may not actually be different.
Intuitively, the difference between the
\emph{true} risk distribution and the \emph{observed}
risk distribution represents a certain ``information loss.''
Optimistically,
if we could somehow improve the informativeness of the predicted scores
to reflect the underlying populations more accurately, then
the resulting selection rule might
exhibit less disparity between $A$ and $B$
in both treatment and impact.

Concretely,
suppose we're given a set of predicted risk scores where
the scores in $A$ tend to be much more extreme (towards
$0$ and $1$) than those of $B$.
Differences in the risk score distributions such as these
can arise for one of two reasons: either individuals from $B$
are \emph{inherently} more stochastic and unpredictable than those
in $A$; or somewhere along the risk estimation pipeline,
more information was lost about $B$ than about $A$.
Understanding which story is true can be
challenging, if not impossible. Still, in cases where we
can reject the hypothesis that certain individuals
are inherently less predictable than others,
the fundamental question to ask is how to
recover the lost information in our predictions.
In this work, we provide tools to answer this question.

\vspace{-11pt}
\paragraph{Refinements.}
Here, we give a technical highlight of the notion of a \emph{refinement} and the role
refinements serve in improving fairness.
In Section~\ref{sec:info}, we introduce the concept of
\emph{information content}, $I(z)$ which gives a global
measure of how informative a calibrated predictor $z$ is over the population of individuals; intuitively, as $I(z)$ increases, the uncertainty in a typical individual's outcome decreases.
The idea that more information in predictions
could lead to better utility or better fairness is not
particularly surprising. Still, this intuition on its own
presents some challenges.  For example, suppose we're concerned
about minimizing the false positive rate (the fraction of the population where $y=0$ that were selected).
Because $I$ is a \emph{global} measure of uncertainty
over all of $\X$, $I(z)$ could be very high due to confidence
about a population $S_0 \subseteq \X$ that is very likely
to have $y=0$ ($z(x)$ close to $0$), even though $z$ gives very little information ($z(x)$ far from $0$ or $1$)
about the rest of $\X$, which consists of a mix of $y=0$ and
$y=1$.  In this case, a predictor $z'$ with less information ($I(z')$), but
better certainty about even a tiny part of the population where $y=1$ ($z(x)$ close to $1$) would
enable lower false positive rates (with nontrivial selection rate).

As such, we need another way to reason about what it means
for one set of predicted risk scores to have ``better information''
than the other.
Refinements provide the key tool for comparing the
information of predictors.
Intuitively, a calibrated predictor $\rho:\X \to [0,1]$
is a refinement of $z$ if $\rho$ hasn't forgotten any
of the information contained by $z$.
Formally, we say that $\rho$ \emph{refines} $z$ if
$\E_{x \sim \X}[\rho(x) \given z(x) = v] = v$;
this definition is closely related to the idea of
calibration in a sense that we make formal in
Proposition~\ref{prop:refine}.

Refinements allow us to 
reason about how information influences
a broad range of quantities of interest in the context of fair prediction.
To give a sense of this,
consider the following lemma, which we use to
prove our main result in Section~\ref{sec:value}, but is also
independently interesting.  The lemma shows that
under any fixed selection rate $\beta = \Pr_{x \sim \X}[f(x) = 1]$,
the true positive rates, false positive rates, and positive
predictive value all improve with a nontrivial refinement.
\begin{lemma*}
If $\rho$ is a refinement of $z$, then for all selection
rates $\beta \in [0,1]$,
\begin{align*}
\TPR^{\rho}(\beta)\ge \TPR^z(\beta),&&
\FPR^{\rho}(\beta)\le \FPR^z(\beta),&&
\PPV^{\rho}(\beta)\ge \PPV^z(\beta).
\end{align*}
\end{lemma*}
Intuitively, the lemma shows that by improving information
through refinements, mutliple key fairness quantities improve simultaneously.
Leveraging this lemma and other properties of refinements
and calibration, we show that for many different ways
a decision-maker might choose their ``optimal'' selection rule,
the ``quality'' of the selection rule improves under refinements.
We highlight this lemma as one example of the broad applicability
of the refinement concept in the context of fair prediction.

\vspace{-11pt}
\paragraph{Perspective.}
Disparities in the information content of risk scores may arise for many reasons.
The present work clarifies how disparities across groups at early stages of the decision-making pipeline may contribute to disparities in the downstream decisions.
In particular, differences in the availability or quality of training data as well as optimization procedures that are tailored for performance in the majority population could contribute to information loss in the minority.
The present work highlights the importance of auditing existing risk score predictors for information content across groups, and demonstrates that obtaining informative calibrated predictions can
improve fair selection rules, even when the fairness desiderata are based on parity.

\vspace{-11pt}
\paragraph{Organization.}
The manuscript is structured as follows:
Section~\ref{sec:prelim} establishes notation and
covers the necessary preliminaries; Section~\ref{sec:info}
provides the technical framework for measuring information
in predictors; Section~\ref{sec:value} demonstrates
how improving information content improves the resulting
fair selection rules; and Section~\ref{sec:blm} describes
the \texttt{merge} algorithm for combining and refining
multiple predictors.  We conclude with a brief discussion
of the context of this work and some directions for future
investigation.

\subsection{Related Works}

The influential work of
\cite{fta} provided two observations that are of
particular relevance to the present work.
First, \cite{fta} emphasized the pitfalls of
hoping to achieve ``fairness through blindness''
by censoring sensitive information during prediction.
Second, the work highlighted how enforcing
broad-strokes demographic parity conditions
-- even if desired or expected in fair outcomes --
is insufficient to imply fairness.
Our results can be viewed as providing further
evidence for these perspectives.
As discussed earlier, understanding the ways in which
fair \emph{treatment} can fail to provide fair \emph{outcomes}
\cite{goel,delayed}
provided much of the motivation for this work.
For a comprehensive overview of the literature
on the growing list of approaches to fairness
in prediction systems, we recommend the
recent encyclopedic survey of \cite{mitchell}.

A few recent works have (implicitly or explicitly)
touched on the relationship between information and fairness.
\cite{chen} 
argues that discrimination may arise in prediction
systems due to disparity in predictive power;
they advocate for addressing discrimination
through data collection.
Arguably, much of the work on fairness in online prediction
\cite{joseph2016fairness} can be seen as
a way to gather information while maintaining fairness.
Recently, issues of information and fairness were also
studied in unsupervised learning tasks \cite{fairpca}.
From the computational economics literature,
\cite{kleinberg2018algorithmic} presents a simple
planning model that draws similar qualitative
conclusions to this work, demonstrating the
significance of trustworthy information as a key factor in
algorithmic fairness.

The idea that better information improves the lender's ability to make useful and fair predictions may seem intuitive under our framing.
Interestingly, different framings of prediction and informativeness can lead to qualitatively different conclusions.
Specifically, the original work on delayed impact \cite{delayed} suggests that some forms of misestimation (i.e.\ loss of information) may reduce the potential for harm from applying parity-based fairness notions.
In particular, if the lender's predictor $z$ is miscalibrated in a way that underestimates the quality of a group $S$, then increasing the selection rate beyond the global utility-maximizing threshold may be warranted.
In our setting, because we assume that the lender's predictions are calibrated, this type of systematic bias in predictions cannot occur, and more information always improves the resulting selection rule.
This discrepancy further demonstrates the importance of group calibration to our notion of information content.

Other works \cite{juba1,juba2} have investigated the role of
\emph{hiding} information through strategic
signaling. 
In such settings, it may be strategic for a group to hide information about individuals in order to increase the overall selection rate for the group.
These distinctions highlight the fact that understanding exactly the role of information in ``fair'' prediction is subtle and also depends on the exact environment of decision-making.
We further discuss how to interpret our theorems as well as the importance of faithfully translating fairness desiderata into mathematical constraints/objectives in Section~\ref{sec:discussion}.

The present work can also be viewed as
further investigating the tradeoffs between
calibration and parity.
Inspired by investigative reporting on the ``biases''
of the COMPAS recidivism prediction system \cite{propublica},
the incompatability of calibration and parity-based
notions of fairness has received lots of attention
in recent years \cite{chouldechova,kmr,pleiss}.
Perhaps counterintuitively,
our work shows how to leverage properties of calibrated predictors to
improve the disparity of the eventual decisions.
At a technical level, our techniques are similar in flavor to
the those of \cite{multi}, which investigates
how to strengthen group-level calibration as a notion
of fairness; we discuss further connections to
\cite{multi} in Section~\ref{sec:discussion}.

Outside the literature on fair prediction, our notions of information content and refinements are related to other notions from the fields of online forecasting and information theory.
In particular, the idea of refinements was first introduced in \cite{degroot}.
The concept of information content of calibrated predictions is related to ideas
from the forecasting literature \cite{gneiting2007probabilistic,gneiting2007strictly}, including \emph{sharpness} and \emph{proper scoring rules} \cite{brier}.
The concept of a refinement of a calibrated predictor can be seen as a special case of Blackwell's informativeness criterion \cite{blackwell,cremer,degroot}.

\section{Preliminaries}
\label{sec:prelim}

\paragraph{Basic notation.}
Let $\X$ denote the domain of individuals and
$\Y = \set{0,1}$ denote the binary outcome space.
We assume that individuals and their outcomes
are jointly distributed according to
$\D$ supported on $\X \times \Y$.
Let $x,y \sim \D$ denote an independent
random draw from $\D$. For a subpopulation $S \subseteq \X$,
we use the shorthand $x,y \sim \D_S$ to be the data
distribution conditioned on $x \in S$, and $x \sim S$ to denote a
random sample from the marginal
distribution over $\X$
conditioned on membership in $S$.

\vspace{-11pt}
\paragraph{Predictors.}
A basic goal in learning
is to find a \emph{classifier}
$f:\X \to \set{0,1}$ that given $x \sim \X$ drawn
from the marginal distribution over individuals, accurately
predicts their outcome $y$.
One common strategy for binary classification first
maps individuals to a real-valued \emph{score} using a
\emph{predictor} $z:\X \to [0,1]$ and then selects
individuals on the basis of this score.  We denote
by $\supp(z)$ the support of $z$.
We denote by $p^*:\X \to [0,1]$ the
\emph{Bayes optimal predictor}, where
$p^*(x) = \Pr_\D\left[y = 1 \given x\right]$ represents the
inherent uncertainty in the outcome given the individual.
 Equivalently, for each individual $x \in \X$,
their outcome $y$ is drawn independently from $\Ber(p^*(x))$,
the Bernoulli distribution with expectation $p^*(x)$.
While we use $[0,1]$ to denote the codomain of
predictors, throughout this work, we assume that
the set of individuals is finite and hence,
the support of any predictor is 
a discrete, finite subset of the interval.

\vspace{-11pt}
\paragraph{Risk score distributions.}
Note that there is a natural bijection between predictors
and \emph{score distributions}.
A predictor $z$, paired with the marginal distribution over
$\X$, induces a score distribution, which we denote $\S^z$,
supported on $[0,1]$,
where the probability density function is given as
$\S^z(v) = \Pr_{x \sim \X}\left[z(x) = v\right]$.
For a subpopulation $S \subseteq \X$, we denote by
$\S^z_S$ the score distribution conditioned on
$x \in S$.

\vspace{-11pt}
\paragraph{Calibration.}
A useful property of predictors is called
\emph{calibration}, which implies that the scores
can be interpreted meaningfully as the probability that
an individual will result in a positive outcome.
Calibration has been studied extensively in varied
contexts, notably in forecasting and online prediction
(e.g.\ \cite{fv}), and recently as a fairness desideratum
\cite{kmr,pleiss,multi,goel};
the definition we use is adapted from the fairness literature.
\begin{definition}[Calibration]
A predictor $z:\X \to [0,1]$ is \emph{calibrated} on a
subpopulation $S \subseteq \X$ if for all $v \in \supp(z)$,
\begin{equation*}
\Pr_{x,y \sim \D_S}\left[y = 1 \given z(x) = v\right] = v.\label{def:calibration}
\end{equation*}
\end{definition}
For convenience when discussing calibration, we use the notation
$S_{z(x)=v} = \set{x \in S : z(x) =v }$.
Note that we can equivalently
define \emph{calibration with respect to the Bayes optimal predictor},
where $z$ is calibrated on $S$ if for all $v \in \supp(z)$,
$\E_{x \sim S_{z(x)=v}}\left[p^*(x)\right] = v$.
Operationally in proofs, we end up using
this definition of calibration.
This formulation also makes clear that
$p^*$ is calibrated on every subpopulation.

\vspace{-11pt}
\paragraph{Parity-based fairness.}
As a notion of fairness, calibration aims to ensure
similarity between predictions and the true outcome
distribution.  Other fairness desiderata concern
disparity in prediction between subpopulations on the
basis of a sensitive attribute.  For simplicity,
we will imagine individuals are partitioned into two
subpopulations $A,B \subseteq \X$; we will overload
notation and use $\set{A,B}$ to denote
the names of the attributes as well.
We let $\A:\X \to \set{A,B}$ map individuals to their
associated attribute.
The most basic notion of parity is
\emph{demographic parity} (also sometimes called \emph{statistical parity} in the literature),
which states that the selection rate of individuals
should be independent of the sensitive attribute.
\begin{definition}[Demographic parity \cite{fta}]
A selection rule $f:\X \to \set{0,1}$ satisfies
\emph{demographic parity} if
\begin{equation*}
\Pr_{x \sim \X}\left[f(x) = 1 \given \A(x) = A\right]
= \Pr_{x \sim \X}\left[f(x) = 1 \given \A(x) = B\right].
\end{equation*}
\end{definition}
One critique of demographic parity is that the notion
does not take into account the actual qualifications of groups (i.e.\ no dependence on $y$).
Another popular parity-based notion,
called \emph{equalized opportunity},
addresses this criticism by enforcing parity of
\emph{false negative rates} across groups.
\begin{definition}[Equalized opportunity \cite{hps}]
A selection rule $f:\X \to \set{0,1}$ satisfies
\emph{equalized opportunity} if
\begin{equation*}
\Pr_{x,y \sim \D}\left[f(x) = 0 \given y=1,\ \A(x) = A\right]
= \Pr_{x,y \sim \D}\left[f(x) = 0 \given y=1,\ \A(x) = B\right]
\end{equation*}
\end{definition}

In addition to these fairness concepts, the following properties of a selection rule will be useful to track.
Specifically, we define the true positive rate ($\TPR$), false positive rate ($\FPR$),
and positive predictive value ($\PPV$).
\begin{align*}
\TPR(f) &= \Pr_{x,y \sim \D}\left[f(x) = 1 \given y=1\right]\\
\FPR(f) &= \Pr_{x,y \sim \D}\left[f(x) = 1 \given y=0\right]\\
\PPV(f) &= \Pr_{x,y \sim \D}\left[y=1 \given f(x) = 1\right]
\end{align*}

\section{Measuring information in binary prediction}
\label{sec:info}

In this section, we give a self-contained exposition of a formal notion of information content in calibrated predictors.
These notions have been studied extensively in the forecasting literature (see \cite{gneiting2007probabilistic,gneiting2007strictly} and references therein), but are less common in the literature on computational and statistical learning theory.
Our notion of information content can be derived from the Brier scoring rule \cite{brier}.

\label{sec:info:content}
In the context of binary prediction,
a natural way to measure the ``informativeness'' of a
predictor is by the uncertainty in an individual's
outcome given their score.  We quantify this uncertainty
using \emph{variance}.\footnote{
Alternatively, we could
measure uncertainty through Shannon entropy (in fact, any function that admits a Bregman divergence).  The generality of the approach is made clear in \cite{gneiting2007strictly}. In Appendix~\ref{app:entropy},
we show that notions of information that arise from Shannon entropy are effectively interchangeable with those that arise from variance.
We elect to work with variance in the main body primarily because it simplifies the analysis in Section~\ref{sec:blm}.}
For $p \in [0,1]$,
the variance of a Bernoulli random variable
with expected value $p$ is given as
$\Var(\Ber(p)) = p \cdot (1-p)$.
Note that variance is a strictly concave function in $p$ and
is maximized at $p=1/2$ and minimized $p\in \set{0,1}$;
that is, a Bernoulli trial with $p=1/2$ is maximally
uncertain whereas a trial with $p=0$ or $p=1$ is
perfectly certain.
Consider a random draw $x,y \sim \D$.
If $z$ is a calibrated predictor, then given $x$ and $z(x)$,
the conditional distribution over $y$ follows a Bernoulli
distribution with expectation $z(x)$.
This observation suggests the following defintion.
\begin{definition}[Information content]
Suppose for $S \subseteq \X$, $z:\X\to[0,1]$ is calibrated on $S$.
The \emph{information content} of $z$ on $S$ is given as
\begin{equation*}
I_S(z) = 1-4\cdot\E_{x \sim S}\left[z(x)(1-z(x))\right].
\end{equation*}
\end{definition}
For a calibrated $z$, we use $I(z) = I_\X(z)$ to
denote the ``information content of $z$''.
The factor $4$ in the definition of information content
acts as a normalization factor such that $I(z) \in [0,1]$.
At the extremes, a perfectly informative predictor has information
content $1$, whereas a calibrated predictor that always
outputs $1/2$ has $0$ information.

This formulation of information content as uncertainty
in a binary outcome
is intuitive in the context of binary classification.
In some settings, however, it may be more
instructive to reason about risk score distributions
directly.  A conceptually different approach
to measuring informativeness of a risk score
distribution might track the uncertainty in the true
(Bayes optimal) risk, given the predicted risk score.

Consider a random variable $P^*_{z(x)=v}$ that takes
value $p^*(x)$ for $x$ sampled from the individuals
with score $z(x) = v$; that is, $P^*_{z(x)=v}$ equals the true risk for an individual sampled amongst those receiving predicted risk score $v$.
Again, we could measure the uncertainty
in this random variable by tracking its variance
given $z(x) = v$;
the higher the variance, the less information the risk score
distribution $\S_z$ provides about the true risk score
distribution $\S_{p^*}$.
Recall, for a predictor $z$ that is calibrated on $S \subseteq \X$
and score $v \in [0,1]$, we let $S_{z(x)= v}= \set{x \in S : z(x)=v}$.
Consider the variance in $P^*_v$ given as
\begin{align*}
\Var\left[P^*_{z(x)=v}\right] &=
\Var_{x \sim S_{z(x) = v}}\left[p^*(x)\right]\\
&= \E_{x \sim S_{z(x)=v}}\left[\left(p^*(x) - v\right)^2\right].
\end{align*}
We define the \emph{information loss} by taking an
expectation of this conditional variance
over the score distribution induced by $z$.
\begin{definition}[Information loss]
For $S \subseteq \X$, suppose a predictor $z:\X \to [0,1]$ is calibrated on $S$.
The \emph{information loss} of $z$ on $S$ is given as
\begin{equation*}
L_S(p^*;z) = 4 \cdot \E_{\substack{v \sim \S^z_S\\x \sim S_{z(x)=v}}}\left[\left(p^*(x) - v\right)^2\right]
\end{equation*}
\end{definition}
Again, the factor $4$ is simply to normalize the
information loss into the range $L(p^*;z) \in [0,1]$.
This loss is maximized when $p^*$ is a $50$:$50$ mix
of $\set{0,1}$ but $z$ always predicts $1/2$;
the information loss is minimized for $z = p^*$.
We observe that this notion of information loss
is actually proportional to the expected squared error
of $z$ with respect to $p^*$; that is,
\begin{align*}
\E_{\substack{v \sim \S_S^z\\x \sim S_{z(x)=v}}}\left[\left(p^*(x) - v\right)^2\right]
&= \sum_{v \in \supp(z)}\S_S^z(v)
\cdot \E_{x \sim S_{z(x)=v}}\left[\left(p^*(x) - v\right)^2\right]\\
&= \E_{x \sim S}\left[\left(p^*(x) - z(x)\right)^2\right]
\end{align*}
Thus, for calibrated predictors,
we can interpret the familiar squared loss between a
predictor and the Bayes optimal predictor
as a notion of information loss.

\vspace{-11pt}
\paragraph{Connecting information content and information loss.}

As we defined them, information content and information
loss seem like conceptually different ways to measure
uncertainty in a predictor.  Here, we show that they
actually capture the same notion.  In particular, we can
express information loss of $z$ as the difference in
information content of $p^*$ and that of $z$. 
\begin{proposition}
\label{prop:infoloss}
Let $p^*:\X \to [0,1]$ denote the Bayes optimal predictor.
Suppose for $S \subseteq \X$, $z:\X \to [0,1]$ is calibrated on $S$.
Then
\begin{equation*}
\label{eqn:infoloss}
L_S(p^*;z) = I_S(p^*) - I_S(z).
\end{equation*}
\end{proposition}
\begin{proof}
The proof follows by expanding the information loss and
rearranging so that, assuming $z$ is calibrated, terms
cancel.
As notational shorthand, let $S_v = S_{z(x)=v}$
and let $\bar{p} = \E_{x \sim S}[p^*(x)] = \E_{x \sim S}[z(x)]$.
Thus, we can rewrite the information loss $L_S(p^*;z)$
as follows.
\begin{align*}
4 \cdot \E_{x \sim S}\left[\left(p^*(x) - z(x)\right)^2\right]
&= 4\cdot \E_{x \sim S}\left[p^*(x)^2 + z(x)^2 - 2p^*(x)\cdot z(x)\right]\\
&=4 \cdot \sum_{v \in \supp(z)}\S_S^z(v)\cdot
\left(\E_{x \sim S_v}\left[p^*(x)^2\right] + v^2 - 2\E_{x \sim S_v}\left[p^*(x)\right]\cdot v\right)\\
&=4 \cdot \sum_{v \in \supp(z)}\S_S^z(v)\cdot
\E_{x \sim S_v}\left[p^*(x)^2 - v^2\right]\label{eqn:infos:cal}\addtag\\
&=4 \cdot \E_{x \sim S}\left[p^*(x)^2 - z(x)^2\right]\\
&= \left(1 - 4 \cdot \E_{x \sim S}\left[\bar{p} - p^*(x)^2\right]\right)
- \left(1 - 4 \cdot \E_{x \sim S}\left[\bar{p}-z(x)^2\right]\right)\\
&= I_S(p^*) - I_S(z)
\end{align*}
where (\ref{eqn:infos:cal}) follows because
$\E_{x \sim S_v}\left[p^*(x)\right] = v$
under calibration.
\end{proof}

Because the information loss is a nonnegative quantity,
Proposition~\ref{prop:infoloss} also formalizes the
inutition that the Bayes optimal predictor is the most
informative predictor; for $z \neq p^*$, $I(z) < I(p^*) \le 1$.
Ideally, in order to evaluate the information disparity
across groups, we would compare
the information lost from $p^*$ to $z$ across $A$ and $B$.
But because the definition of information
loss depends on the true score distribution $p^*$,
in general, it's impossible to directly compare the loss. 
Still, if we believe that $p^*(x)$
is similarly distributed across $x \sim A$ and $x \sim B$,
then measuring the information contents $I_A(z)$ and $I_B(z)$
-- properties of the \emph{observed} risk scores --
allows us to directly compare the loss.

\subsection{Incorporating information via refinements}
\label{sec:info:refine}
We have motivated the study of informativeness in prediction
with the intuition that as information content improves,
so too will the resulting fairness and utility of the
decisions derived from the predictor.  Without further
assumptions, however, this line of reasoning
turns out to be overly-optimistic.
For instance, consider a setting where 
the expected utility of lending to individuals
is positive if $z(x) > \tau$ for some fixed threshold $\tau$ (and negative
otherwise).
In this case, information about individuals whose $p^*(x)$
is significantly below $\tau$ is not especially useful.
Figure~\ref{fig:pmfs} gives an example
of two predictors, each calibrated to the same $p^*(x)$,
where $I(z') > I(z)$, but $z$ is preferable.
At a high-level, the example exploits the fact that
information content $I(z)$ is a global property of $z$,
whereas the quantities that affect the utility and fairness directly,
like $\PPV$ or $\TPR$ are \emph{conditional} properties. 
Even if $I(z') > I(z)$, it could be that $z'$ has lost
information about an important subpopulation compared to $z$,
compensating with lots of information about the
unqualified individuals.
\begin{figure}[t!]
    \centering
        \vspace{-22pt}
        \includegraphics[width=0.3\textwidth]{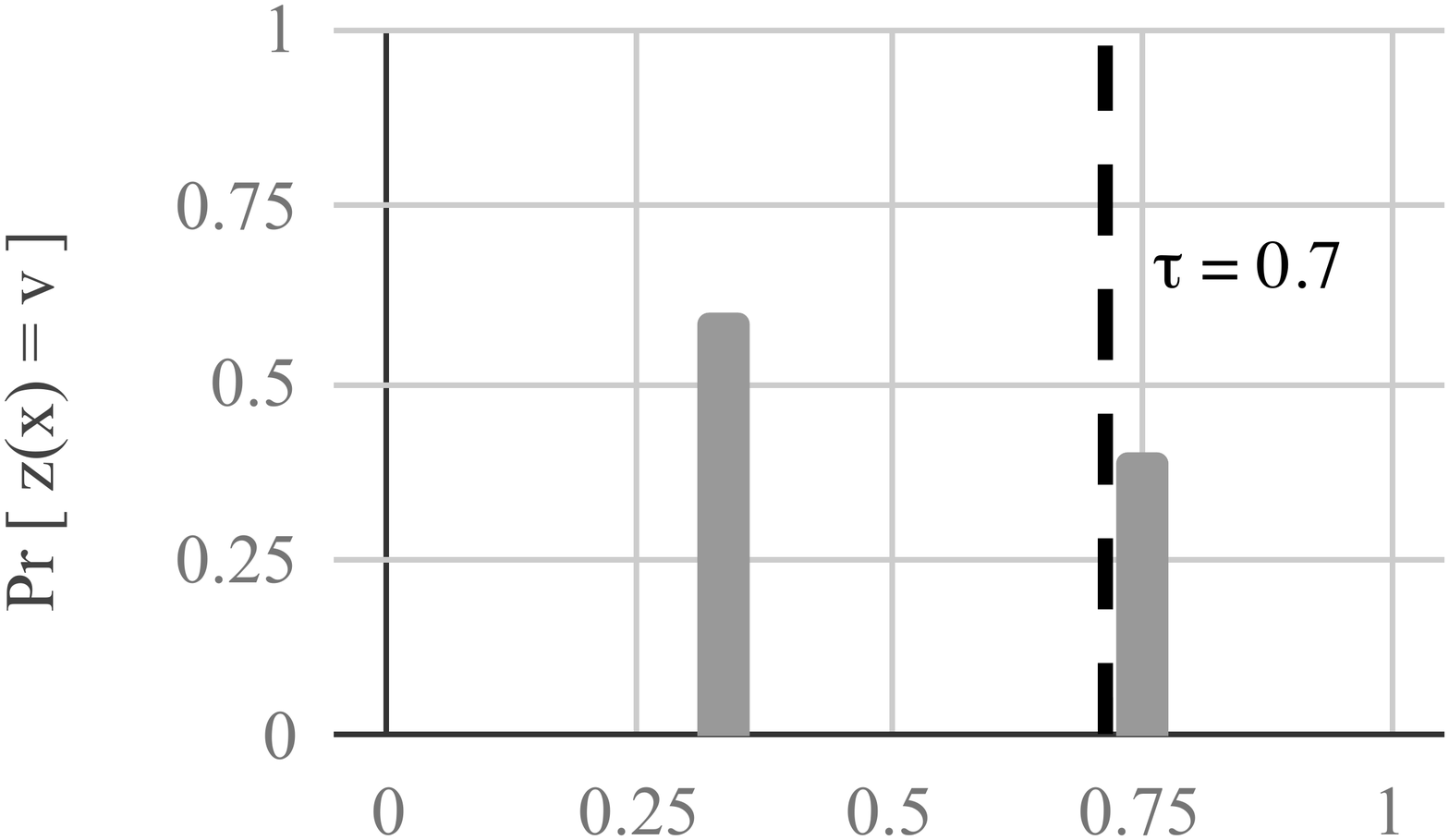}\hspace{8pt}
        \includegraphics[width=0.3\textwidth]{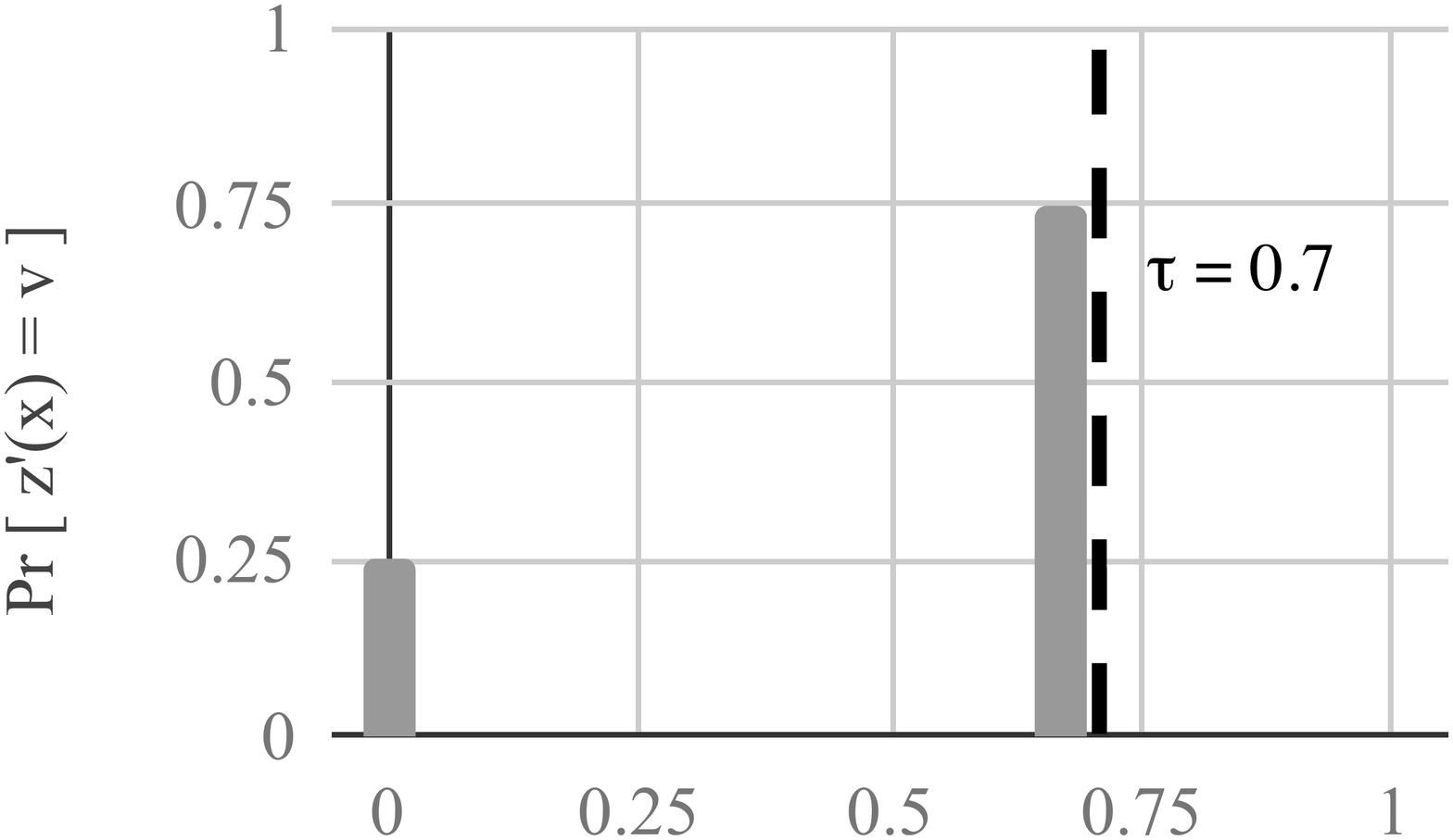}
        \vspace{-11pt}
    \caption{Comparing information content directly is insufficient to compare predictors' utility.
    \emph{Two predictors $z,z':\X \to [0,1]$ are each calibrated
    to $p^*:\X \to \set{0,1}$ with $\E\left[p^*(x)\right] = 1/2$.
    In $z$, $\Pr[z(x) = 1/3] = 3/5$ and $\Pr[z(x) = 3/4] = 2/5$;
    in $z'$, $\Pr[z(x) = 0] = 1/4$ and $\Pr[z(x) = 2/3] = 3/4$.
    $I(z') > I(z)$, but $z$ achieves better utility than $z'$ whenever
    $2/3 < \tau < 3/4$.}}\label{fig:pmfs}
    \vspace{-11pt}
\end{figure}

Still, we would like to characterize ways in which more information
is definitively ``better.''  Intuitively, more information is
better when we don't have to give up on the information
in the current predictor, but rather refine the information
contained in the predictions further.
The following definition, equivalent to a notion proposed in \cite{degroot}, formalizes this idea.
\begin{definition}[Refinement]
For $S \subseteq \X$,
suppose $z,z':\X \to [0,1]$ are calibrated on $S$.
$z'$ is a \emph{refinement} of $z$ on $S$ if for all $v \in \supp(z)$,
\begin{equation*}
\E_{x \sim S_{z(x)=v}}\left[z'(x)\right] = v.
\end{equation*}
\end{definition}
That is, we say that $z'$ \emph{refines} $z$ if $z'$ maintains
the same expectation over the level sets $S_{z(x)=v}$.
To understand why this property makes sense in the context of maintaining
information from $z$ to $z'$, suppose the property was violated:
that is, there is some $v \in \supp(z)$ such that
$\E_{x \sim S_{z(x)=v}}[z'(x)]\neq v$.  This disagreement provides
evidence that $z$ has some consistency with the true risk
that $z'$ is lacking; because $z$ is calibrated,
$\E_{x \sim S_{z(x)=v}}[z(x)] = v = \E_{x \sim S_{z(x)=v}}[p^*(x)]$.
In other words, even if $z'$ has greater information content,
it may not be consistent with the content of $z$.

Another useful perspective on refinements is through measuring
the information on each of the sets $S_{z(x)=v}$.
Restricted to $S_{z(x)=v}$, $z$ has minimal information
content -- its predictions are constant --
whereas $z'$ may vary.
Because $z'$ is calibrated and maintains the expectation over $S_{z(x)=v}$,
we can conclude that $I_{S_{z(x)=v}}(z) \le I_{S_{z(x)=v}}(z')$ for each of
the partitions.

This perspective highlights the importance of requiring \emph{calibration} in the definition of refinements.
Indeed, because a refinement is a calibrated predictor, refinements cannot make arbitrary distinctions in predictions, so any additional distinctions on the level sets of the original predictor must represent true variability in $p^*$.
We draw attention to the similarity between
the definition of a refinement and the definition of calibration.
In particular, if $z'$ is a refinement of $z$, then $z$ is not only
calibrated with respect to $p^*$, but also with respect to $z'$; stated differently, $p^*$ is a refinement of every calibrated predictor.
Indeed, one way to interpret a refinement is as a
``candidate'' Bayes optimal predictor.
Carrying this intuition through, we note that the only property
of $p^*$ we used in the proof of Proposition~\ref{prop:infoloss}
is that it is a refinement of a calibrated $z$.  Thus,
we can immediately restate the proposition in terms of generic
refinements.
\begin{proposition}\label{prop:refine}
Suppose for $S \subseteq \X$, $z,z':\X \to [0,1]$ are calibrated on $S$.
If $z'$ refines $z$ on $S$, then
$L_S(z';z) = I_S(z') - I_S(z)$.
\end{proposition}
This characterization further illustrates the notion that
a refinement $z'$ could plausibly be the true risk
given the information in the current predictions $z$.
In particular, because $L_S(z';z)> 0$, we get $I_S(z') > I_S(z)$
for any refinement $z'\neq z$.

In the context of fair prediction, we want to
ensure that the information content on specific protected
subpopulations does not decrease. Indeed, in this case,
it may be important to ensure that the predictions are
refined, not just overall, but also on the sensitive
subpopulations.
In Figure~\ref{fig:group:refine},
we illustrate this point by showing two predictors
$z,z':\X\to[0,1]$ that are each calibrated on two
subpopulations $A,B \subseteq \X$;
$z'$ refines $z$ on $\X$ overall,
but $z'$ loses information about the subpopulation $A$.
This negative example highlights the importance of
incorporating all the information available (e.g.\ group
membership), not only at the time of decision-making,
but also along the way when developing predictors;
it serves as yet another rebuke of the approach
of ``fairness through blindness'' \cite{fta}.
\begin{figure}[t!]
    \centering
    \vspace{-22pt}
    \hspace{-22pt}\includegraphics[height=1.3in]{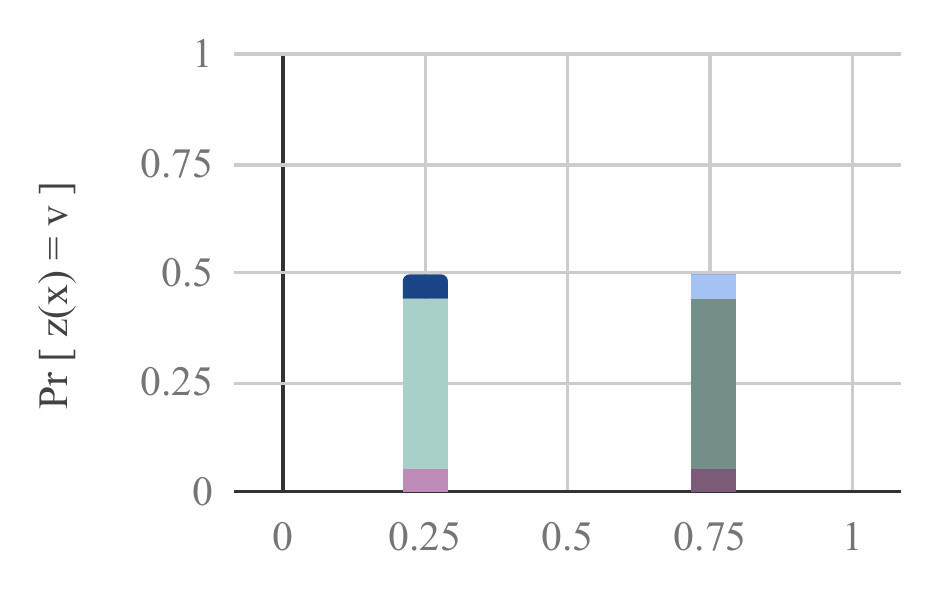}\hspace{-8pt}\includegraphics[height=1.3in]{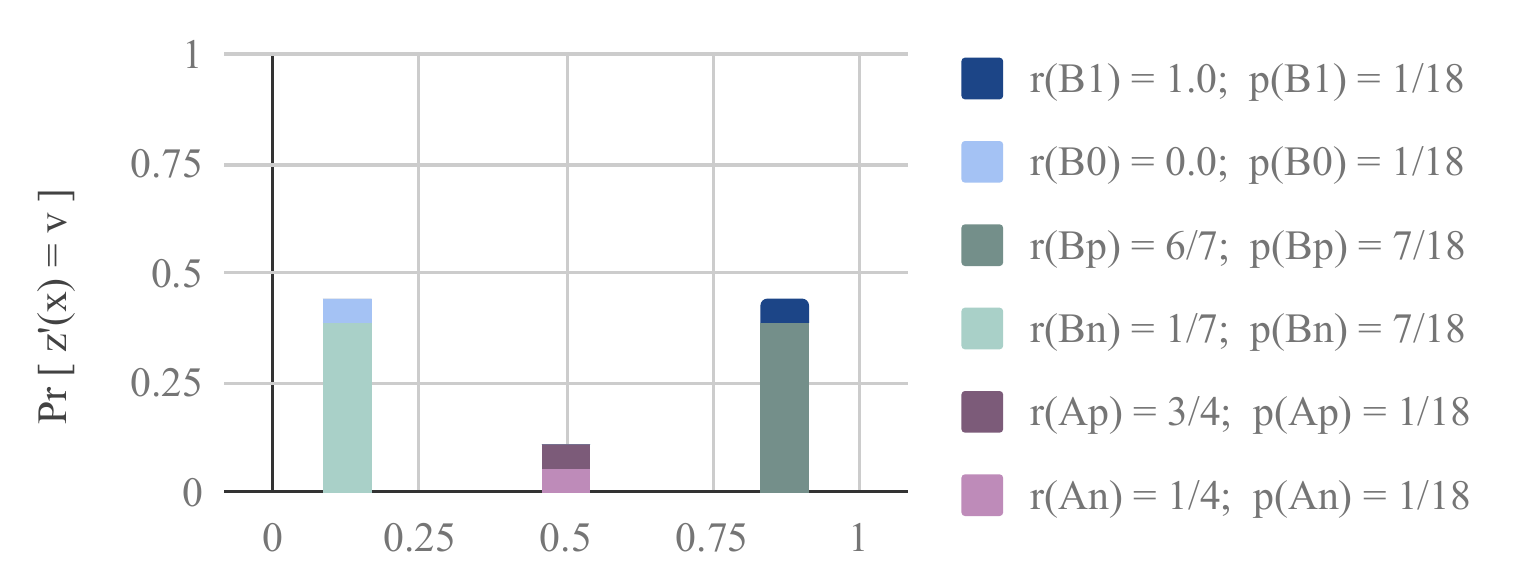}
    \caption{Per-group refinement is necessary to maintain information for each group.\label{fig:group:refine}
    \emph{
    Let $A = A_n \cup A_p$ and $B = B_n \cup B_p \cup B_0 \cup B_1$ where $r(S) = \E_{x \sim S}[p^*(x)]$ and $p(S) = \Pr_{x \sim \X}[x \in S]$.
    The two predictors $z,z':\X \to [0,1]$ are each calibrated
    on $A$ and $B$.  Note that $z'$ refines $z$ overall,
    but has lost all information about $A$.}}\label{fig:cal}
    \vspace{-11pt}
\end{figure}

\section{The value of information in fair prediction}
\label{sec:value}

In this section, we argue that
reasoning about the information content
of calibrated predictors provides a lens
into understanding how to improve the
utility and fairness of predictors,
even when the eventual fairness desideratum
is based on parity.
We discuss a prediction setting
based on that of \cite{delayed} where a \emph{lender}
selects individuals to give loans from a pool of \emph{applicants}.
While we use the language of predicting creditworthiness,
the setup is generic and can be applied to diverse prediction
tasks.
\cite{delayed} introduced a notion of ``delayed
impact'' of selection policies, which models
the potential negative impact on communities
of enforcing parity-based fairness as a constraint.
We revisit the question of delayed impact as part of
a broader investigation of the role of information
in fair prediction.
We begin with an overview of the prediction
setup.  Then, we prove our main result:
refining the underlying predictions used
to choose a selection policy results in an
improvement in utility, parity, or impact
(or a combination of the three).

\subsection{Fair prediction setup}
\label{sec:value:setup}
When deciding how to select qualified individuals,
the lender's goal is to maximize some expected
utility. 
Specifically, the \emph{utility function}
$u:[0,1] \to [-1,1]$ specifies the lender's
expected utility from an individual based on their
score and a fixed threshold\footnote{Assuming such an
affine utility function is equivalent
to assuming that the lender
receives $u_+$ from repayments, $u_-$ from defaults,
and $0$ from individuals that do not receive loans.
In this case, the expected utility
for score $p$ is $pu_+ + (1-p)u_- = c_u \cdot u(p)$ for some
constant $c_u$.  A similar rationale
applies to the individuals' impact function.
} $\tau_u \in [0,1]$ as given in (\ref{eqn:utilityimpact}).
When considering delayed impact,
we will measure the expected impact per subpopulation.
The \emph{impact function} $\ell:[0,1] \to [-1,1]$
specifies the expected benefit to an individual from
receiving a loan
based on their score and a fixed threshold $\tau_\ell$ also
given in (\ref{eqn:utilityimpact}).
\begin{align}
\label{eqn:utilityimpact}
u(p) = p - \tau_u& &\ell(p) = p - \tau_\ell
\end{align}
\cite{delayed} models risk-aversion of the lender by
assuming that $\tau_u > \tau_\ell$; that is, by choosing
accepting individuals with $z(x) \in (\tau_\ell,\tau_u)$,
the impact on subpopulations may improve beyond the
lender's utility-maximizing policy.

In this setup, we allow the lender to pick a
(randomized) group-sensitive selection policy
$f:[0,1] \times \set{A,B} \to [0,1]$ that selects
individuals on the basis of a predicted score
and their sensitive attribute.
That is, the selection policy makes decisions
about individuals
via their score according to some calibrated predictor
$z:\X \to [0,1]$ and their sensitive attribute
$\A:\X \to \set{A,B}$; for every individual
$x \in \X$, the probability that $x$ is selected
is given as $f(z(x),\A(x))$.

We will restrict our attention to
\emph{threshold policies};
that is, for sensitive attribute $A$ (resp.,\ $B$),
there is some $\tau_A \in [0,1]$, such that $f(v,A)$
is given as $f(v,A) = 1$ if $v > \tau_A$,
$f(v,A) = 0$ if $v < \tau_A$ and $f(v,A) = p_A$
for $v=\tau_A$,
where $p_A \in [0,1]$ is a probability used to
randomly break ties on the threshold.
The motivation for focusing on threshold policies is
their intuitiveness, widespread use, computational
efficiency\footnote{Indeed, without the restriction
to threshold policies, many of the \emph{information-theoretic}
arguments become easier at the expense of \emph{computational cost}.
As $z'$ is a refinement of $z$, we can always simulate decisions derived from $z$ given $z'$, but in general, we cannot do this efficiently.}.
The restriction to threshold policies
is justified formally in \cite{delayed}
by the fact that the optimal decision rule
in our setting can be specified as a threshold policy
under both demographic parity and equalized opportunity.

Given this setup, we can write the expected utility $U^z(f)$
of a policy $f$ based on a calibrated predictor
$z$, that is calibrated on both subpopulations, 
$A$ and $B$, as follows.
\begin{equation}
\label{eqn:exputility}
U^z(f)
= \sum_{S\in \{A,B\}}\Pr_{x \sim \X}\left[x \in S\right] \cdot
\left(\sum_{v \in \supp(z)}
\R^z_S(v)
\cdot f(v,S) \cdot u(v)\right)
\end{equation}
Recall, $\S^z_S(v) = \Pr_{x \sim S}\left[z(x) = v\right]$.

Similarly, the expected impact over the subpopulations
$S \in \set{A,B}$ are given as
\begin{equation}
\label{eqn:expimpact}
\Imp_S^z(f)
=\sum_{v\in \supp(z)} {\S^z_S(v)\cdot f(v,S)\cdot \ell(v)}
\end{equation}
Often, it may make sense to constrain the net impact to each group
as defined in (\ref{eqn:expimpact}) to be positive, ensuring that
the selection policies do not do harm as in \cite{delayed}.

The following quantities will be of interest to the
lender when choosing a selection policy $f$ as a
function of $z$.
First, the lender's overall utility $U(f)$ is given
as in (\ref{eqn:exputility}).
In the name of fairness, the lender may also be
concerned about the disparity of a number of quantities.
We will show below that these quantities can be written as a linear function of
the selection rule. 
In particular, for $S \in \set{A,B}$
demographic parity, which serve as our running example, compares the
\emph{selection rates} $\beta_S = \Pr_{x \sim S}[x \text{ selected}]$,
\begin{equation}
\beta^z_S(f)=\sum_{v\in\supp(z)}\R^z_S(v)\cdot f(v,S).
\end{equation}
We may also be concerned about 
comparing the true positive rates (equalized opportunity)
and false positive rates.
Recall, $\TPR = \Pr[x \text{ selected} \given y = 1]$ and $\FPR = \Pr[x \text{ selected} \given y=0]$; in this context, we can rewrite these quantities as follows.
\begin{gather}
\TPR^z_S(f)=\frac{1}{r_S}\cdot\sum_{v \in \supp(z)}\R^z_S(v)
\cdot f(v,S)\cdot v\\
\FPR^z_S(f)=\frac{1}{1-r_S}\cdot\sum_{v\in \supp(z)}
\R^z_S(v)\cdot f(v,S)\cdot (1-v),
\end{gather}
where $r_S$ represents the base rate of the subpopulation $S$; that is, $r_S=\Pr_{(x,y)\sim D_S}[y=1]$.
Another quantity we will track is the positive predictive value, $\PPV = \Pr[y=1 \given x \text{ selected}]$.
\begin{equation}
\PPV^z_S(f)=\frac{1}{\beta_S^z(f)}\cdot\left(\sum_{v\in \supp(z)}\R^z_S(v)\cdot f(v,S)\cdot v\right)
\end{equation}
Note that $\PPV^z_S(f)$ is not a linear function of $f(v,S)$ values, but as we never use positive predictive values directly in the optimizations 
for choosing a selection policy (or in a parity-based fairness definition), the optimizations are still a linear program. For notational convenience, we may drop the superscript of these quantities when $z$ is clear from the context.

\subsection{Refinements in the service of fair prediction}\label{sec:improvref}

Note that all of the quantities described in Section~\ref{sec:value:setup} can be written
as linear functions of $f(v,S)$.
Given a fixed predictor $z:\X \to [0,1]$,
we can expand the quantities of interest;
in particular, we note that the linear functions over $f(v,S)$ can be rewritten as linear functions over $z(x)$, where the quantities depend on $x$ only through the predictor $z$.
In this section, we show how refining the predictor used for determining the selection rule can improve
the utility, parity, and impact of the optimal selection rule.
By the observations above, we can formulate a generic
policy-selection problem as a linear program where
$z$ controls many coefficients in the program.
When we refine $z$, we show that the value of the
program increases.
Recalling that different contexts may call for
different notions of fairness, we consider a
number of different linear programs the lender
(or regulator) might choose to optimize.
At a high-level, the lender can choose to maximize
utility, minimize disparity, or maximize positive
impact on groups, while also maintaining some
guarantees over the other quantities.

We will consider selection policies
given a fixed predictor $z:\X \to [0,1]$.
Note that the parity-based fairness desiderata we consider
are of the form $h^z_A(f)=h^z_B(f)$ for some
$h\in\set{\beta,\TPR,\FPR}$; rather than requiring
equality, we will consider the disparity
$\card{h^z_A(f) - h^z_B(f)}$ and in some cases,
constrain it to be less than some constant $\eps$.
We also use $t_i,t_u$ to denote lower bounds
on the desired impact and utility, respectively.
For simplicity's sake, we assume that $B$ is the
``protected'' group, so we only enforce the positive
impact constraint for this group; more generally,
we could include an impact constraint for each group.
Formally, we consider the following constrained
optimizations.

\begin{center}
\begin{tabular}{c c c}
\refstepcounter{opt}\label{opt:utility}{\bf Optimization~\theopt}
&
\refstepcounter{opt}\label{opt:fairness}{\bf Optimization~\theopt} 
&
\refstepcounter{opt}\label{opt:impact}{\bf Optimization~\theopt} \\

$\max_f~U^z(f)$
&$\min_f~\card{h^z_A(f) - h^z_B(f)}$
&$\max_f~\Imp^z_B(f)$\\

s.t.~~$\Imp^z_B(f)\ge t_{i}$
& s.t.~~$\Imp^z_B(f)\ge t_{i}$
& s.t.~~$U^z(f)\ge t_{u}$\\

$\card{h^z_A(f)-h^z_B(f)} \le \eps$
&$U^z(f)\ge t_{u}$
&$\card{h^z_A(f)-h^z_B(f)} \le \epsilon$\\

~~~~(\emph{Utility Maximization})~~~~
&
~~~~(\emph{Disparity minimization})~~~~
&
~~~~(\emph{Impact Maximization})~~~~

\end{tabular}
\end{center}

\begin{lemma}
Let $h \in \set{\beta, \TPR, \FPR}$.
Given a calibrated predictor $z:\X \to [0,1]$,
Optimization \ref{opt:utility},\ref{opt:fairness}, and \ref{opt:impact} are linear programs in the variables $f(v,S)$ for $v \in \supp(z)$ and $S \in \set{A,B}$.
Further, for each program, there is an optimal solution $f^*$ that is a threshold policy.
\end{lemma}
We sketch the proof of the lemma.
The fact that the optimizations are linear programs follows immediately from the observations that each quantity of interest is a linear function in $f(v,S)$.
The proof that there is a threshold policy $f^*$ that achieves the optimal value in each program is similar to the proof of Theorem~\ref{thm:improve} given below.
Consider an arbitrary (non-threshold) selection policy $f_0$; let $h_{S,0} = h_S(f_0)$.
The key observation is that for the
fixed value of $h_{S,0}$, there is some other threshold policy $f$ where $h^z_S(f) = h_{S,0}$ and $U^z(f) \ge U(f_0)$ and $\Imp^z_S(f) \ge \Imp_S(f_0)$.
Leveraging this observation, given any non-threshold optimal selection policy, we can construct a threshold policy, which is also optimal.

We remark that our analysis applies even if considering
the more general linear maximization:
\begin{center}
\begin{tabular}{c}
\refstepcounter{opt}\label{opt:generic}{\bf Optimization~\theopt}\\
$\max_f~~
\lambda_U \cdot U^z(f)
+ \lambda_I\cdot \Imp_B^z(f)
- \lambda_\beta\cdot \card{h^z_A(f)-h^z_B(f)}$
\end{tabular}
\end{center}
for any fixed $\lambda_U, \lambda_I, \lambda_\beta\ge 0$.\footnote{In 
particular, each of Optimizations~\ref{opt:utility},
\ref{opt:fairness}, and \ref{opt:impact} can be
expressed as an instance of Optimization~\ref{opt:generic}
by choosing $\lambda_U, \lambda_I, \lambda_\beta$ to
be the optimal dual multipliers for each program.  We
note that the dual formulation actually gives an alternate
way to derive results from \cite{delayed}.
Their main result can be restated as saying that
there exist distributions of scores such that
the dual multiplier on the positive impact
constraint in Optimization~\ref{opt:utility}
is positve; that is, without this
constraint, the utility-maximizing policy will do
negative impact to group $B$.}
In other words, the arguments hold no matter the
relative weighting of the value of utility, disparity, and
impact.

\vspace{-11pt}
\paragraph{Improving the cost of fairness.}
We argue that in all of these optimizations,
increasing information through refinements of the
current predictor on both the subpopulations $A$ and $B$
improves this value of the program.
We emphasize that this conclusion is true
for all of the notions of parity-based fairness we
mentioned above.  Thus, independent of the exact
formulation of fair selection that policy-makers deem appropriate,
information content is a key factor in determining
the properties of the resulting selection rule.
We formalize this statement in the following theorem.

\begin{theorem}\label{thm:improve}
Let $z,z':\X\rightarrow[0,1]$ be two predictors that are 
calibrated on disjoint subpopulations $A,B \subseteq \X$.
For any of the Optimization~\ref{opt:utility},~\ref{opt:fairness},~\ref{opt:impact},~\ref{opt:generic} and their corresponding
fixed parameters,
let $\textrm{OPT}(z)$ denote their optimal value under predictor $z$.
If $z'$ refines $z$ on $A$ and $B$,
then $\textrm{OPT}(z') \ge \text{OPT}(z)$ for Optimization~\ref{opt:utility},~\ref{opt:impact},~\ref{opt:generic} and  $\textrm{OPT}(z') \le \text{OPT}(z)$ for Optimization~\ref{opt:fairness}.
\end{theorem}
One way to understand Theorem~\ref{thm:improve}
is through a ``cost of fairness'' analysis.
Focusing on the utility maximization setting,
let $U^*$ be the maximum unconstrained utility
achievable by the lender given the optimal predictions
$p^*$.  Let $\textrm{OPT}(z)$ be the optimal value
of Optimization~\ref{opt:utility} using predictions
$z$; that is, the best utility a lender can achieve
under a parity-based fairness constraint ($\eps=0$)
and
positive impact constraint ($t_i=0$). If we take the cost of
fairness to be the difference between these optimal
utilities, $U^* - \textrm{OPT}(z)$, then
Theorem~\ref{thm:improve} says that by refining
$z$ to $z'$, \emph{the cost of fairness decreases with
increasing informativeness}; that is,
$U^* - \textrm{OPT}(z) \ge U^* - \textrm{OPT}(z')$.
This corollary of Theorem~\ref{thm:improve}
corroborates the idea that in some cases the \emph{high perceived
cost} associated with requiring fairness might actually be due to the
\emph{low informativeness} of the predictions in 
minority populations. No matter what the true $p^*$
is, this cost will decrease as we increase information
content by refining subpopulations.

For $S\in\{A,B\}$, we use $\TPR_S^{z}(\beta)$ to denote the true positive rate of the threshold policy with selection rate $\beta$ for the subpopulation $S$ while using the predictor $z$\footnote{Given a predictor, there is a bijection between selection rates and threshold policies.}. Similarly, $\PPV_S^z(\beta)$, $\FPR_S^z(\beta)$ are defined. 
The following lemma, which plays a key role in each proof,
shows that refinements broadly improve selection policies
across these three statistics of interest.
\begin{lemma} \label{cl:improv}
If $z'$ is a refinement of $z$ on subpopulations
$A$ and $B$, then for $S\in\{A,B\}$, for all $\beta \in [0,1]$,
\begin{align*}
\TPR_S^{z'}(\beta)\ge \TPR_S^z(\beta),&&
\FPR_S^{z'}(\beta)\le \FPR_S^z(\beta),&&
\PPV_S^{z'}(\beta)\ge \PPV_S^z(\beta).
\end{align*}
\end{lemma}

In particular, the proof of Theorem \ref{thm:improve} crucially uses the fact that the positive predictive values, true positive rates, and false positive rates improve for \emph{all} selection rates.
Leveraging properties of refinements, the improvement across all selection rates guarantees improvement for any fixed objective.
As we'll see, the proof actually tells us more: for \emph{any} selection policy using the predictor $z$, there exists a threshold selection policy that uses the refined predictor $z'$ and \emph{simultaneously} has utility, disparity, and impact that are no worse than under $z$.
In this sense, increasing informativeness of predictors through refinements is an effective strategy for improving selection rules across a wide array of criteria.
Still, we emphasize the importance of identifying fairness desiderata and specifying them clearly when optimizing for a selection rule.

For instance, suppose the selection rule is selected by constrained utility maximization with a predictor $z$ and with a refined predictor $z'$.
It is possible that the \emph{optimal} selection policy under the refinement $z'$ will have a lower quantitative impact than the optimal policy under the original predictor $z$ (while still satisfying the impact constraint).
If maintaining the impact above a certain threshold is desired, then this should be specified clearly in the optimization used for determining the selection rule.
We defer further discussion of these issues to Section \ref{sec:discussion}.

\vspace{-11pt}
\paragraph{Proofs.}
Next, we prove Lemma~\ref{cl:improv} and Theorem~\ref{thm:improve}.

\begin{proofof}{Lemma~\ref{cl:improv}}
Note that for a fixed selection rate $\beta$, $\PPV$
is maximized by picking the top-most $\beta$ fraction of
the individuals ranked according to $p^*$,
i.e.\ a threshold policy that selects a $\beta$-fraction
of the individuals using the Bayes optimal predictor $p^*$.
Similarly, for a fixed selection rate, the $\TPR$ and $\FPR$ values are also optimized under a threshold selection policy that uses the Bayes optimal predictor $p^*$.

Recall, we can interpret a refinement as a
``candidate'' Bayes optimal predictor.  In particular,
because $z'$ refines $z$ over $A$ and $B$,
we know that $z$ is calibrated not only with respect
to the true Bayes optimal predictor $p^*$, but also with respect to the refinement $z'$ on both subpopulations.
Imagining a world in which
$z'$ is the Bayes optimal predictor, the $\PPV$, $\TPR$, and $\FPR$ must be no worse under a threshold policy derived from $z'$ compared to that of $z$ by the initial observation.
Thus, the lemma follows.
\end{proofof}

Using Lemma~\ref{cl:improv}, we are ready to prove Theorem~\ref{thm:improve}.

\begin{proofof}{Theorem~\ref{thm:improve}}
Let $f$ be any threshold selection policy under the predictor $z$.
Using $f$, we will construct a selection policy $f'$ that uses the refined score distribution $z'$ such that
where $U^{z'}(f') \ge U^{z}(f)$, $\Imp^{z'}_B(f') \ge \Imp^{z}_B(f)$,
and $h^{z'}_A(f')=h^z_A(f)$ and $h^{z'}_B(f')=h^z_B(f)$. Here,
$h\in\{\beta,\TPR,\FPR\}$ specifies the parity-based fairness definition being used. 
Thus, taking $f$ to be the optimal solution to any of the 
Optimizations~\ref{opt:utility},~\ref{opt:fairness},~\ref{opt:impact}, 
or~\ref{opt:generic}, we see that $f'$ is a feasible solution to the same optimization and has the same 
or a better objective value
compared to $f$. Therefore, after optimization, objective values can only get better.

In words, we are saying that refined predictors allow us to get better utility and impact as the original predictor while keeping the parity values the same (for e.g., while keeping the selection rates the same in both subpopulations).

We separately construct $f'$ for each fairness definition ($h$) as follows:
\begin{enumerate}
\item (Demographic Parity) $h=\beta$:

For $S \in \set{A,B}$, let $\beta_S = \beta^z_S(f)$ be the selection
rate of $f$ in the population $S$.
Let $f'$ be the threshold policy that uses the predictor $z'$
and achieves selection rates $\beta_A$ and $\beta_B$ in the subpopulations
$A$ and $B$, respectively.
By Lemma~\ref{cl:improv}, $\PPV_S^{z'}(\beta_S)\ge \PPV^{z}_S(\beta_S)$
for $S \in \set{A,B}$. The utility of the policy $f'$ can be written as 
\begin{align*}
U(f')&= \sum_{S\in \{A,B\}}\Pr_{x \sim \X}\left[x \in S\right] \cdot
\left(\sum_{v \in \supp(z')}
\R^{z'}_S(v)
\cdot f'(v,S) \cdot v- \sum_{v \in \supp(z')}
\R^{z'}_S(v)
\cdot f'(v,S)\cdot \tau_u\right)\\
&=\sum_{S\in \{A,B\}}\Pr_{x \sim \X}\left[x \in S\right] \cdot
\left(\beta_S\cdot (\PPV^{z'}_S(\beta_S)-\tau_u)\right)\\
&\ge \sum_{S\in \{A,B\}}\Pr_{x \sim \X}\left[x \in S\right] \cdot
\left(\beta_S\cdot (\PPV^{z}_S(\beta_S)-\tau_u)\right)\\
&=U(f)
\end{align*}
Similarly, we can show that the impact on the subpopulation $B$ under $f'$ is at least as
good as under $f$.

\item (Equalized Opportunity) $h=\TPR$:

Let $(\beta_A,\beta_B)$ be the selection rates of policy $f$ on the subpopulations $A$ and $B$. We know that $\TPR_S^{z'}(\beta_S)\ge \TPR^{z}_S(\beta_S)$ ($S\in\{A,B\}$) through Lemma \ref{cl:improv}. Let $f'$ be the threshold selection policy corresponding to a selection rates of $\beta_S',(S\in\{A,B\})$ such that $\TPR^{z'}_S(\beta_S')= \TPR^{z}_S(\beta_S)$ ($\le  \TPR^{z'}_S(\beta_S)$). As the true positive rates increase with increasing selection rate, $\beta_S'\le \beta_S$. The utility of the policy $f'$ can be written as 
\begin{align*}
U(f')&= \sum_{S\in \{A,B\}}\Pr_{x \sim \X}\left[x \in S\right] \cdot
\left(\sum_{v \in \supp(z')}
\R^{z'}_S(v)
\cdot f'(v,S) \cdot v- \sum_{v \in \supp(z')}
\R^{z'}_S(v)
\cdot f'(v,S)\cdot \tau_u\right)\\
&=\sum_{S\in \{A,B\}}\Pr_{x \sim \X}\left[x \in S\right] \cdot
\left(r_S\cdot \TPR^{z'}_S(\beta_S')-\beta_S'\cdot \tau_u\right)\\
&\ge\sum_{S\in \{A,B\}}\Pr_{x \sim \X}\left[x \in S\right] \cdot
\left(r_S\cdot \TPR^{z}_S(\beta_S)-\beta_S\cdot \tau_u\right)\\
&=U(f)
\end{align*}
Similarly, we can show that the impact on the subpopulation $B$ under $f'$ is at least as
good as under $f$.

\item (Equalized False Positive Rate) $h=\FPR$.

Let $(\beta_A,\beta_B)$ be the selection rates of policy $f$ on the subpopulations $A$ and $B$. We know that $\FPR_S^{z'}(\beta_S)\le \FPR^{z}_S(\beta_S)$ ($S\in\{A,B\}$) through Lemma \ref{cl:improv}. Let $f'$ be the threshold selection policy corresponding to a selection rates of $\beta_S',(S\in\{A,B\})$ such that $\FPR^{z'}_S(\beta_S')= \FPR^{z}_S(\beta_S)$ ($\ge \FPR_S^{z'}(\beta_S)$). As the false postive rates increase with increasing selection rate, $\beta_S'\ge \beta_S$. The utility of the policy $f'$ can be written as 
\begin{align*}
U(f')&= \sum_{S\in \{A,B\}}\Pr_{x \sim \X}\left[x \in S\right] \cdot
\left(\sum_{v \in \supp(z')}
\R^{z'}_S(v)
\cdot f'(v,S) \cdot v- \sum_{v \in \supp(z')}
\R^{z'}_S(v)
\cdot f'(v,S)\cdot \tau_u\right)\\
&=\sum_{S\in \{A,B\}}\Pr_{x \sim \X}\left[x \in S\right] \cdot
\left(\beta_S'-(1-r_S)\cdot \FPR^{z'}_S(\beta_S')-\beta_S'\cdot \tau_u\right)\\
&=\sum_{S\in \{A,B\}}\Pr_{x \sim \X}\left[x \in S\right] \cdot
\left(\beta_S'\cdot (1-\tau_u)-(1-r_S)\cdot \FPR^{z'}_S(\beta_S')\right)\\
&\ge\sum_{S\in \{A,B\}}\Pr_{x \sim \X}\left[x \in S\right] \cdot
\left(\beta_S\cdot (1-\tau_u)-(1-r_S)\cdot \FPR^{z}_S(\beta_S)\right)\\
&=U(f)
\end{align*}
Similarly, we can show that the impact on the subpopulation $B$ under $f'$ is at least as
good as under $f$.

\end{enumerate}
This completes the proof of the theorem.
\end{proofof}

\section{A mechanism for refining predictors}
\label{sec:blm}

In this section, we outline a mechanism for obtaining
refinements of predictors.  We start by describing
an algorithm, \texttt{merge}, which given two calibrated
predictors, produces a new refined predictor that 
incorporates the information from both predictors
(in a sense we make formal).
For notational convenience, we describe how to
refine a predictor over $\X$.  The arguments here
extend easily to refining over a partition of $\X$
by refining each part separately.
We discuss the generality of the approach
at the end of the section and elaborate on
the possibility of refinements on overlapping subpopulations
briefly in Section~\ref{sec:discussion}.

Given a predictor $z$, we can evaluate the information
content $I(z)$ directly; estimating the information
loss $L(p^*;z)$, however, is generally impossible.  Indeed,
without assumptions on the structure of $p^*$ or
the ability to sample every individual's outcome repeatedly
(and independently),
we cannot reason about information-theoretic quantities like
$I(p^*)$.
Still, supposing that the information loss of $z$ is
sufficiently large,
we would like to be able to certify this fact and ideally,
bring the loss down.

The most obvious way to demonstrate that a predictor $z$ can be refined
would be to exhibit some calibrated $q:\X \to [0,1]$ such that $I(q) > I(z)$.
That said, expecting that we could obtain such a $q$ seems
to sidestep the question of how to update a predictor to
improve its information content.
Even if we were able to obtain some $q$ where
$I(q) > I(z)$, it is not clear that $q$ would be a
``better'' predictor.  In particular, $q$ might
contain \emph{different} information than $z$;
recall that such examples motivated the definition
of a refinement in the first place.
Still, intuitively, if $q$ is not a refinement
of $z$ and contains very different information than $z$,
then $q$ should be useful in identifying ways to improve the
informativeness of $z$.  Further, this intuition
does not seem to rely on the fact that
$I(q) > I(z)$; as long as $q$ contains information
that isn't ``known'' to $z$, then incorporating
the information into $z$ should reduce the
information loss.

To formalize this line of reasoning,
first, we need to make precise
what we mean when we say that $q$ is far from
refining $z$.  Recalling the definition of a
refinement, consider the logical negation of the
statement that ``$q$ refines $z$.''
\begin{equation*}
\neg\left(\forall v \in \supp(z):\
\E_{x \sim \X_{z(x)=v}}\left[q(x)\right] = v\right)
\iff \exists v \in \supp(z):\ 
\E_{x \sim \X_{z(x)=v}}\left[q(x)\right] \neq v.
\end{equation*}
Extending this logical formulation, we define the
following divergence to capture quantitatively
how far $q$ is from refining $z$.
\begin{definition}[Refinement distance]
Let $q,z:\X \to [0,1]$ be calibrated predictors.
The \emph{refinement distance} from $z$ to $q$
is given as
\begin{equation*}
D_R(z;q) = \sum_{v \in \supp(z)} \R^z(v)
\cdot \card{\E_{x \sim \X_{z(x)=v}}\left[q(x)\right] - v}.
\end{equation*}
\end{definition}
Note that $D_R(z;q)$
is not symmetric; in particular, if $q$ refines
$z$ and contains more information $I(q) > I(z)$,
then $D_R(z;q) = 0$, but $D_R(q;z) > 0$.
Intuitively, the refinement distance averages
the refinement ``disagreements'' over all values
in the support.  We show that, under calibration,
these disagreements can be reconciled to improve
the overall information content.
With the notion of refinement distance in place,
we can state the main algorithmic result --
a simple algorithm for aggregating the information
of multiple calibrated predictors
into a single calibrated predictor.
\begin{theorem}
\label{thm:merge}
Given two calibrated predictors $q,z:\X \to [0,1]$,
Algorithm~\ref{alg:merge}
produces a new calibrated predictor $\rho:\X \to [0,1]$
such that $\rho$ is a refinement of both $z$ and $q$.
Further, 
$I(\rho) > \max\set{I(z) + 4\cdot D_R(q;z)^2,I(q) + 4\cdot D_R(z;q)^2}$.
\end{theorem}
We state the theorem generally, making no assumptions
about $D_R(q;z)$ or $D_R(z;q)$.  In particular,
as we alluded to earlier, if $q$ already refines $z$,
then $D_R(z;q) = 0$, so there will be no information gain.
Algorithm~\ref{alg:merge}, which we refer to as
\texttt{merge}, describes the procedure.
We describe the algorithm in the statistical query
model, where we assume query access to aggregate
statistics about $p^*(x)$.  At the end of the section,
we discuss the sample complexity needed to answer such
statistical queries accurately.
The \texttt{merge} algorithm builds a new calibrated predictor $\rho$ from $q$ and $z$ by considering the set of individuals who receive $q(x) = u$ and $z(x) = v$ for each $u \in \supp(q)$ and $v \in \supp(z)$.
For each of these sets, the merged predictor adjusts the prediction to have the correct expectation.
The proof of Theorem~\ref{thm:merge} follows from a standard potential function analysis; 
further, the sample complexity needed to answer the
statistical queries accurately is bounded.

\begin{figure}[ht!]
{\refstepcounter{algorithm} \label{alg:merge}{\bf Algorithm~\thealgorithm:}}
\texttt{merge(z,q)}

\fbox{\parbox{\textwidth}{
\vspace{4pt}
{\bf Given:} $z,q:\X \to [0,1]$ calibrated predictors\\
{\bf Output:} $\rho:\X \to [0,1]$ a refinement of $z$ and $q$
\begin{itemize}
\item Let $\Z = \set{\X_{z(x)=v} : v \in \supp(z)}$ 
\item Let $\mathcal{Q} = \set{\X_{q(x)=u} : u \in \supp(q)}$
\item For $Z_v \in \Z$ and $Q_u \in \mathcal{Q}$:

\begin{itemize}
  \item $\X_{vu} = Z_v \cap Q_u$
  \item $\rho(x) \gets \E\limits_{x \sim \X_{vu}}\left[p^*(x)\right]$
\end{itemize}
\end{itemize}
\vspace{-8pt}
}
}
\end{figure}

We break the proof of Theorem~\ref{thm:merge} into two
lemmas.  First, note that the \texttt{merge} procedure is symmetric
with respect to $q$ and $z$.  Thus, any statements
about the output $\rho$ in terms of one of the inputs $z$
will also be true with respect to the input $q$.
\begin{lemma}
\label{lem:merge:refine}
Let $\rho$ be the output of Algorithm~\ref{alg:merge}
on two calibrated predictors $z,q:\X \to [0,1]$ as input.
$\rho$ is a refinement of $z$.
\end{lemma}
\begin{proof}
We use the notation established in Algorithm~\ref{alg:merge}.
In particular, we refer to the conditional score distribution
$\R^q_{Z_v}$ where $\R^q_{Z_v}(u) = \Pr_{x \sim \X}\left[q(x)=u \given z(x) = v\right]$.
Consider the expectation of $\rho$ over the level sets of
$z$, $\set{Z_v : v \in \supp(z)}$.
\begin{align}
\E_{x \sim Z_v}\left[\rho(x)\right]
&=\sum_{u \in \supp(q)}\R^q_{Z_v}(u)
\cdot \E_{x \sim \X_{vu}}\left[\rho(x)\right]\notag\\
&=\sum_{u \in \supp(q)}\R^q_{Z_v}(u)
\cdot \E_{x \sim \X_{vu}}\left[\E_{x\sim\X_{vu}}[p^*(x)]\right]\label{eqn:lem:assign}\\
&=\sum_{u \in \supp(q)}\R^q_{Z_v}(u)
\cdot \E_{x \sim Z_{v}}\left[p^*(x) \given q(x)=u\right]\label{eqn:lem:condition}\\
&=\E_{x \sim Z_v}\left[p^*(x)\right]\notag
\end{align}
where (\ref{eqn:lem:assign}) follows by the assignment rule of
$\rho(x)$ for $x \in \X_{vu}$; and (\ref{eqn:lem:condition})
follows from exploiting $\X_{vu} = \set{x \in Z_v : q(x)=u}$.
By the calibration of $z$, we see the final expression is
equal to $v$.
This argument is independent of $v$, so for all $v \in \supp(z)$,
$\E_{x \sim Z_v}\left[\rho(x)\right]
= v$; thus, by definition, $\rho$ refines $z$.
\end{proof}
The next lemma shows that the information of $\rho$
increases based on the refinement distance.
Note that in combination Lemma~\ref{lem:merge:refine} and
Lemma~\ref{lem:infogain} prove Theorem~\ref{thm:merge}.
\begin{lemma}
\label{lem:infogain}
Let $\rho$ be the output of Algorithm~\ref{alg:merge}
on two calibrated predictors $z,q:\X \to [0,1]$ as input.
Then,
\begin{equation*}
I(\rho) \ge I(z) + 4\cdot D_R(q;z)^2.
\end{equation*}
\end{lemma}
\begin{proof}
We can lower bound the resulting information content of
$\rho$ by reasoning about the difference $I(\rho)-I(z)$.
Note that by Lemma~\ref{lem:merge:refine}, $\rho$ is a
refinement of $z$; further, by Proposition~\ref{prop:refine},
we can express $I(\rho) - I(z)$ as $L(\rho;z)$.
Expanding the information loss, we can lower bound the
gain in information, which shows the lemma.
\begin{align}
\frac{1}{4}\cdot L(\rho;z)&= \E_{x \sim \X}\left[(\rho(x) - z(x))^2\right]\notag\\
&= \sum_{u \in \supp(q)}\R^q(u)
\cdot \E_{x \sim Q_u}\left[\left(z(x) - \rho(x)\right)^2\right]\notag\\
&\ge \sum_{u \in \supp(q)}\R^q(u)
\cdot \left(\E_{x \sim Q_u}\left[z(x) - \rho(x)\right]\right)^2\label{eqn:alg:jensen1}\\
&= \sum_{u \in \supp(q)}\R^q(u)
\cdot \left(\E_{x \sim Q_u}\left[z(x)\right] - u\right)^2\label{eqn:alg:refine}\\
&\ge \left(\sum_{u \in \supp(q)}\R^q(u)
\cdot \card{\E_{x \sim Q_u}\left[z(x)\right] - u}\right)^2\label{eqn:alg:jensen2}\\
&\ge D_R(q;z)^2\label{eqn:alg:epsfar}
\end{align}
where (\ref{eqn:alg:jensen1}) follows by
Jensen's Inequality; (\ref{eqn:alg:refine}) notes that
$\E_{x\sim Q_u}[\rho(x)] = u$ because
$q$ is calibrated and $\rho$ refines $q$;
(\ref{eqn:alg:jensen2}) applies Jensen's inequality again;
and (\ref{eqn:alg:epsfar}) follows by
the definition of refinement distance.
\end{proof}

One appealing consequence of Theorem~\ref{thm:merge} is
that the number of times a predictor needs to be significantly
updated is bounded.  In particular, suppose we are merging
two calibrated predictors $q,z$; for any $\eta \ge 0$,
we'll say the operation is an $\eta$-merge if $\min\set{D_R(z;q),D_R(q;z)} \ge \eta$.
In this case, the information content of the result predictor
will increase by at least $\Omega(\eta^2)$ and any given predictor can be $\eta$-merged at most $O(1/\eta^2)$ times.
In other words, as long as the information being combined is not
too similar, then the number of such merge updates is bounded.

\vspace{-11pt}
\paragraph{Interpreting the updates.}

As described, the \texttt{merge}
algorithm takes two different calibrated predictors and
combines them into a refinement.  In settings where the
lender wants to combine predictions from different sources,
this algorithmic model is naturally well-motivated.
Still, there are other settings where the \texttt{merge}
algorithm can be applied.
One natural
way we can specify new information content is by giving
the predictor an additional feature.  Specifically,
consider some a boolean feature $\phi:\X \to \set{0,1}$.
We define the predictor $q_\phi:\X \to [0,1]$
to be $q_\phi(x) =
\E_{x' \sim \X}\left[p^*(x') \given \phi(x') = \phi(x)\right]$.
This predictor gives the expected value over the set
of individuals where $\phi(x) = 1$ (resp.,\ $\phi(x) = 0$);
thus, the predictor is calibrated.
Merging $z$ with $q_\phi$
incorporates the information in
the boolean feature $\phi$ into the predictions of $z$.  In particular, the
information content framework gives us a way to reason
about the marginal informativeness of individual boolean features;
the greater the difference between
$\E_{x \sim \phi^{-1}(0)}[p^*(x)]$ and
$\E_{x \sim \phi^{-1}(1)}[p^*(x)]$, the more informative.

This perspective is particularly salient when we think
external regulation of predictors.  For example, consider
some subpopulation $S \subseteq \X$.  One way to provide
evidence that $S$ is experiencing discrimination under $z$
would be to demonstrate that merging the predictor $q_{S}$
into $z$ significantly changes the information content.
This could occur because the quality of
individuals in $S$ are consistently underestimated
by $z$ or because the quality of individuals in
$\X\setminus S$ are consistently overestimated.

\vspace{-11pt}
\paragraph{Implementing the merge from samples.}

While we presented the \texttt{merge} algorithm assuming access
to a statistical query oracle, in practice, we want to
estimate the necessary statistical queries from data.
We assume that the predictions are discretized to precision
$\alpha$; that is, we represent the interval $[0,1]$ as
$[\alpha/2,3\alpha/2,\hdots, 1-\alpha/2]$.
We argue that the number of samples needed to obtain
accurate statistics in this model is bounded as follows.
\begin{proposition}\label{prop:samples}
Consider an execution of Algorithm~\ref{alg:merge} such that
$\min_{\X_{vu}}{\Pr_{x \sim \X}\left[x \in \X_{vu}\right]} \ge \gamma$.
Then from $m \ge \tilde{\Omega}\left(\frac{\log(1/\delta)}{\gamma\alpha^2}\right)$ random samples from $\D$, with probability $1-\delta$,
every statistical query
can be answered with some
$q_{vu}$ such that
$\card{q_{vu} - \E_{x \sim \X_{vu}}\left[p^*(x)\right]} < \alpha/2$.
\end{proposition}
\begin{proof}
The argument follows by a standard uniform convergence argument.
To start, note that there are at most $1/\alpha^2$ queries to
answer.
Suppose we have $t$ random samples $(x_i,y_i) \sim \D$ conditioned
on $x_i \in \X_{vu}$ for all $i \in [t]$.
Let $q_{vu}$ denote the empirical expectation
of $y_i$'s on these $t$ samples over $\X_{vu}$; that is,
\begin{equation*}
q_{vu} \triangleq \frac{1}{t} \sum_{i=1}^t y_i.
\end{equation*}
By Hoeffding's inequality:
\begin{equation*}
\Pr\left[~\card{q_{vu} - \E_{x \sim \X_{vu}}[y \given x]} > \alpha \right]
\le 2e^{2t\alpha^2}.
\end{equation*}
If $t \ge \Omega\left(\frac{\log(2/\delta\alpha^2)}{\alpha^2}\right)$,
then the probability of failure is at most $\alpha^2\delta/2$.
Union bounding over the queries, the probability of failure is
at most $\delta/2$.

Thus, we need to bound the sample complexity needed to
hit each $\X_{vu}$ at least $t$ times.
By assumption, each $\X_{vu}$ has density at least $\gamma$.
Thus, for each $\X_{vu}$,
the probability that a random sample from $(x,y) \sim \D$
has $x \in \X_{vu}$ is at least $\gamma$.
If we take $\Omega(\log(2t/\delta)/\gamma)$ such samples,
then the probability every sample misses $\X_{vu}$ is at most
$\delta/2t$.
Thus, if we take $\Omega(t\log(2t/\delta)/\gamma)$ samples,
each $\X_{vu}$ will have at least $t$ samples with probability
at least $1-\delta/2$.

By union bound, the proposition follows.
\end{proof}

\section{Discussion}
\label{sec:discussion}

In this work, we identify information disparity as a potential source of discrimination in prediction tasks.
We provide an introduction to key concepts of information content and loss, and show how improving the information content of predictions improves the resulting fairness of the downstream decisions.
In particular, our results show when a lender does not have sufficient statistical or computational resources to learn a predictor that achieves small squared error across all significant subpopulations, issues of unfairness may arise due to differential information loss.

The information content of a predictor $z$ can be significantly larger on the majority population $S$ than the minority $T$ for a number of reasons.
\begin{itemize}
\item Despite optimal predictions, the individuals in $S$ are inherently more predictable than those in $T$; i.e.\ $z \approx p^*$ and $I_S(p^*) > I_T(p^*)$.
If this (controversial) hypothesis is true, there may be no way to improve the predictions further, and some degree of disparity may be unavoidable.
Note that in general this condition cannot be verified from data.
Still, the assumption that $I(z) = I(p^*)$ can be falsified by finding a way to give more informative predictions.
\item Nontrivial information loss has occurred in $z$ on $T$ compared to $S$; i.e.\ $L_S(p^*;z) < L_T(p^*;z)$.
Such information loss could result from purely information theoretic issues (features used for prediction are not sufficiently expressive in $T$), a mix of informational and computational issues (not enough data from the minority to learn a predictor from a sufficiently rich hypothesis class), or purely computational issues (suboptimal learning in $T$ due to optimization for $S \cup T$).
Each source of information disparity has a different own solution (collecting better features, collecting more data, re-training with awareness of the population $T$, respectively), but the tools we present for reasoning about information content apply broadly.
\end{itemize}
As such, improving the information content of predictions may require collecting additional features or data.
Once collected, the \texttt{merge} procedure provides a relatively inexpensive way of incorporating new information into the predictions while retaining certain ``quality" of the selection rule.
In practice, when the sensitive group $T$ is known, it may make sense to simply retrain the prediction model with awareness of $T$.

\vspace{-11pt}
\paragraph{Overlapping subpopulations.}
The proposed \texttt{merge} algorithm provides a simple
and efficient approach for producing refinements in applications where
the sensitive populations are well understood. Often, as highlighted
in \cite{multi,kearns2017preventing,ftba},
the subpopulations in need of protection may be hard to
anticipate.  These recent works have studied notions
of \emph{multi-fairness} that aim to strengthen notions
of group fairness by enforcing statistical constraints
not just overall, but on a rich family of subgroups.
In particular, \cite{multi} introduces a
notion called \emph{multicalibration}, which informally,
guarantees calibration across all subpopulations specified by
a given set system $\C$.  We observe that the
guarantees of multicalibration can be reinterpreted
in the language of refinements: a predictor 
that is multicalibrated with respect to $\C$ is
simultaneously a refinement for all $q_c:\X \to [0,1]$
where $q_c(x) =
\E_{x' \sim \X}\left[p^*(x') \given c(x') = c(x)\right]$
for every $c \in \C$.  In this sense, multicalibration
may be an effective approach to improving information
across subpopulations, when the groups that might be
experiencing discrimination are unknown or overlapping.
An interesting question is whether some of the
analysis in the present work can be applied to
understand better the connections between
multicalibration and the work of \cite{kearns2017preventing}
which studies the notion of rich subgroup fairness under
demographic parity and equalized opportunity.

\paragraph{Choosing fairness constraints.}
Understanding precisely the guarantees
of the specified fairness constraints is particularly important for
interpreting the results of Section~\ref{sec:value}.
In particular, refining predictions is guaranteed to
improve the value of the \emph{stated program}.
We emphasize the importance of faithfully translating
fairness desiderata into mathematical requirements.

For instance,
suppose policy-makers want to increase representation for
historically-disenfranchised populations.
An appealing -- but misguided -- translation of this goal would require
demographic parity; intuitively, if the lender is required
to have equalized selection rates across groups,
they might increase the selection rate
in the minority to match that of the majority.
Still, demographic parity only requires parity of selection
rates which could also be achieved by reducing the selection in
the majority.
Further, refinements under demographic parity constraints,
might cause the selection rates in the minority
to \emph{decrease}.  By increasing information, we
might uncover that fewer individuals are actually above
the tolerable risk than the previous predictions suggested;
as such, fewer individuals might be deemed qualified for a loan.

The framework proposed in Section~\ref{sec:value} is
compatible with a variety of constraints and objectives,
including explicitly lower bounding the group selection rates.
Thus, increasing the selection rate in a given population can always be achieved by directly constraining the selection rule.
An appealing aspect of the framework is that it allows
policy-makers to experiment with constraints
and objectives to understand the downstream effects of
their policies, given the current set of predictions.
For instance, policy-makers can evaluate how lower bounding the selection rate in a group will affect the impact of the policy on this group.
Such experimentation with the programs from Section~\ref{sec:value} may help to guide future policy decisions.

\vspace{-11pt}
\paragraph{Changing environments.}
The present work focuses on a setting where the true
risk scores of the underlying population does not change;
that is, we assumed that the Bayes optimal predictor $p^*$
remains fixed while producing better and better refinements.
In real life, the true risk of individuals may change
as their environment changes, and possibly even \emph{as a result
of the prior decisions made by the lender}, as suggested
by \cite{delayed}.  An exciting direction for future
investigation would study a setting of dynamic $p^*$,
with the goal of ensuring long-term fairness and impact.
A specific challenge is finding an efficient
(in terms of sample and time complexity) procedure
for maintaining calibration when $p^*$ changes over time.
Further, we showed the importance 
of increasing informativeness of predictors for
underrepresented populations, but required access
to random samples from this population.
In settings where random exploration may cause harm
to uncertain populations (e.g.\ by raising the default rate)
how can we improve information without causing the
inherent capabilities of sensitive subpopulations
to deteriorate?

\vspace{-11pt}
\paragraph{Conclusion.}
We reiterate that
the validity of every notion of fairness 
rests on some set of assumptions.
Many approaches to fair classification assume implicitly
that the predicted risk scores represent the true risk scores.
Predicted risk scores are the result of an extensive pipeline
of data collection and computational modeling; when
data is limited for minority populations and modeling is
focused on fidelity in the majority populations, the
resulting predictions may not be appropriately informative
in the minority.
In the case that the differences arise because of
suboptimal predictions, increasing information through
refinements provide a simple but effective approach for
improving the utility, fairness, and impact of the decision-maker's policy.

\paragraph{Acknowledgments.}
The authors thank Cynthia Dwork and Guy N.\ Rothblum for many helpful conversations throughout the development of this work.
We thank Moritz Hardt, Gal Yona, and anonymous reviewers for feedback on earlier versions of the work.

\newpage
\bibliographystyle{alpha}
% {\footnotesize
\bibliography{refs,impact}

\newcommand{\etalchar}[1]{$^{#1}$}
\begin{thebibliography}{KNRW18}

\bibitem[ALMK16]{propublica}
Julia Angwin, Jeff Larson, Surya Mattu, and Lauren Kirchner.
\newblock Machine bias: There's software used across the country to predict
  future criminals. and it's biased against blacks.
\newblock {\em ProPublica}, 2016.

\bibitem[BCZ{\etalchar{+}}16]{wordvecs}
Tolga Bolukbasi, Kai-Wei Chang, James~Y Zou, Venkatesh Saligrama, and Adam~T
  Kalai.
\newblock Man is to computer programmer as woman is to homemaker? debiasing
  word embeddings.
\newblock In {\em Neural Information Processing Systems}, 2016.

\bibitem[BG18]{gendershades}
Joy Buolamwini and Timnit Gebru.
\newblock Gender shades: Intersectional accuracy disparities in commercial
  gender classification.
\newblock In {\em FAT$^*$}, 2018.

\bibitem[Bla53]{blackwell}
David Blackwell.
\newblock Equivalent comparisons of experiments.
\newblock {\em The annals of mathematical statistics}, 1953.

\bibitem[Bri50]{brier}
Glenn~W Brier.
\newblock Verification of forecasts expressed in terms of probability.
\newblock {\em Monthly Weather Review}, 1950.

\bibitem[CG18]{goel}
Sam {Corbett-Davies} and Sharad Goel.
\newblock The measure and mismeasure of fairness: A critical review of fair
  machine learning.
\newblock {\em arXiv preprint 1808.00023}, 2018.

\bibitem[Cho17]{chouldechova}
Alexandra Chouldechova.
\newblock Fair prediction with disparate impact: A study of bias in recidivism
  prediction instruments.
\newblock {\em Big Data}, 2017.

\bibitem[CJS18]{chen}
Irene Chen, Fredrik~D Johansson, and David Sontag.
\newblock Why is my classifier discriminatory?
\newblock {\em Neural Information Processing Systems}, 2018.

\bibitem[Cr{\'e}82]{cremer}
Jacques Cr{\'e}mer.
\newblock A simple proof of blackwell's “comparison of experiments”
  theorem.
\newblock {\em Journal of Economic Theory}, 1982.

\bibitem[DF81]{degroot}
Morris~H DeGroot and Stephen~E Fienberg.
\newblock Assessing probability assessors: Calibration and refinement.
\newblock Technical report, CARNEGIE-MELLON UNIV PITTSBURGH PA DEPT OF
  STATISTICS, 1981.

\bibitem[DHP{\etalchar{+}}12]{fta}
Cynthia Dwork, Moritz Hardt, Toniann Pitassi, Omer Reingold, and Richard~S.
  Zemel.
\newblock Fairness through awareness.
\newblock In {\em ITCS}, 2012.

\bibitem[FK18]{korolovaFB}
Irfan Faizullabhoy and Aleksandra Korolova.
\newblock Facebook's advertising platform: New attack vectors and the need for
  interventions.
\newblock {\em arXiv preprint 1803.10099}, 2018.

\bibitem[FV98]{fv}
Dean~P. Foster and Rakesh~V. Vohra.
\newblock Asymptotic calibration.
\newblock {\em Biometrika}, 1998.

\bibitem[GBR07]{gneiting2007probabilistic}
Tilmann Gneiting, Fadoua Balabdaoui, and Adrian~E Raftery.
\newblock Probabilistic forecasts, calibration and sharpness.
\newblock {\em Journal of the Royal Statistical Society: Series B (Statistical
  Methodology)}, 2007.

\bibitem[GR07]{gneiting2007strictly}
Tilmann Gneiting and Adrian~E Raftery.
\newblock Strictly proper scoring rules, prediction, and estimation.
\newblock {\em Journal of the American Statistical Association}, 2007.

\bibitem[HKRR18]{multi}
{\'{U}}rsula H{\'{e}}bert{-}Johnson, Michael~P. Kim, Omer Reingold, and Guy~N.
  Rothblum.
\newblock Calibration for the (computationally-identifiable) masses.
\newblock {\em ICML}, 2018.

\bibitem[HM19]{fifty}
Ben Hutchinson and Margaret Mitchell.
\newblock 50 years of testing (un)fairness: Lessons for machine learning.
\newblock In {\em FAT$^*$}, 2019.

\bibitem[HPS16]{hps}
Moritz Hardt, Eric Price, and Nathan Srebro.
\newblock Equality of opportunity in supervised learning.
\newblock In {\em Neural Information Processing Systems}, 2016.

\bibitem[ILZ19]{juba2}
Nicole Immorlica, Katrina Ligett, and Juba Ziani.
\newblock Access to population-level signaling as a source of inequality.
\newblock {\em FAT$^*$}, 2019.

\bibitem[JKMR16]{joseph2016fairness}
Matthew Joseph, Michael Kearns, Jamie~H. Morgenstern, and Aaron Roth.
\newblock Fairness in learning: Classic and contextual bandits.
\newblock In {\em Neural Information Processing Systems}, 2016.

\bibitem[KLMR18]{kleinberg2018algorithmic}
Jon Kleinberg, Jens Ludwig, Sendhil Mullainathan, and Ashesh Rambachan.
\newblock Algorithmic fairness.
\newblock In {\em AEA Papers and Proceedings}, 2018.

\bibitem[KMR17]{kmr}
Jon~M. Kleinberg, Sendhil Mullainathan, and Manish Raghavan.
\newblock Inherent trade-offs in the fair determination of risk scores.
\newblock In {\em ITCS}, 2017.

\bibitem[KNRW18]{kearns2017preventing}
Michael Kearns, Seth Neel, Aaron Roth, and Zhiwei~Steven Wu.
\newblock Preventing fairness gerrymandering: Auditing and learning for
  subgroup fairness.
\newblock {\em ICML}, 2018.

\bibitem[KRR18]{ftba}
Michael~P. Kim, Omer Reingold, and Guy~N. Rothblum.
\newblock Fairness through computationally-bounded awareness.
\newblock {\em Neural Information Processing Systems}, 2018.

\bibitem[KRZ19]{juba1}
Sampath Kannan, Aaron Roth, and Juba Ziani.
\newblock Downstream effects of affirmative action.
\newblock {\em FAT$^*$}, 2019.

\bibitem[LDR{\etalchar{+}}18]{delayed}
Lydia~T. Liu, Sarah Dean, Esther Rolf, Max Simchowitz, and Moritz Hardt.
\newblock Delayed impact of fair machine learning.
\newblock In {\em ICML}, 2018.

\bibitem[MPB18]{mitchell}
Shira Mitchell, Eric Potash, and Solon Barocas.
\newblock Prediction-based decisions and fairness: A catalogue of choices,
  assumptions, and definitions.
\newblock {\em arXiv preprint 1811.07867}, 2018.

\bibitem[PRW{\etalchar{+}}17]{pleiss}
Geoff Pleiss, Manish Raghavan, Felix Wu, Jon~M. Kleinberg, and Kilian~Q.
  Weinberger.
\newblock On fairness and calibration.
\newblock In {\em Neural Information Processing Systems}, 2017.

\bibitem[SbF{\etalchar{+}}19]{selbst}
Andrew~D. Selbst, danah boyd, Sorelle~A. Friedler, Suresh Venkatasubramanian,
  and Janet Vertesi.
\newblock Fairness and abstraction in sociotechnical systems.
\newblock In {\em FAT$^*$}, 2019.

\bibitem[STM{\etalchar{+}}18]{fairpca}
Samira Samadi, Uthaipon Tantipongpipat, Jamie~H Morgenstern, Mohit Singh, and
  Santosh Vempala.
\newblock The price of fair pca: One extra dimension.
\newblock In {\em Neural Information Processing Systems}, 2018.

\end{thebibliography}
% }

\appendix

\section{Measuring information through Shannon entropy}
\label{app:entropy}

For completeness, we briefly discuss how to relate
the notions of information defined in Section~\ref{sec:info} to
an analogous notion of information, defined through Shannon
entropy.  In particular, rather than defining
information content in terms of the variance of
$y$ given $z(x)$, we could have defined it in
terms of the Shannon entropy of this random variable.
In particular, the entropy of a Bernoulli random
variable with expectation $p$ is captured by the
binary entropy function $H_2(p)$ where
\begin{equation*}
H_2(p) = -p \cdot \log(p) - (1-p)\cdot \log(1-p).
\end{equation*}
For a calibrated predictor $z:\X \to [0,1]$, let $\IH_S(z) = 1-\E_{x \sim S}\left[H_2(x)\right]$ denote the \emph{entropic information content}, parameterized by the binary entropy (rather than variance).
The binary entropy $H_2(p)$ always upper bounds
the scaled variance $4\cdot p(1-p)$,
with equality at $p\in \set{0,1/2,1}$.
As a consequence, the following inequality holds.
\begin{equation*}
\IH_S(z) \le I_S(z)
\end{equation*}
When information is parameterized by entropy, the corresponding
natural notion of ``information loss'' is parameterized
by the expected KL-divergence.
Specifically, denote by $D_{KL}(p;q)$ the KL-divergence
between two Bernoulli distributions with expectations
$p$ and $q$, respectively, defined as
\begin{equation*}
D_{KL}(p;q) = p\cdot \log\left(\frac{p}{q}\right) + (1-p)\cdot \log\left(\frac{1-p}{1-q}\right).
\end{equation*}
Again, for a calibrated predictor $z:\X \to[0,1]$,
let $\LH_S(p^*;z) = \E_{x \sim S}\left[D_{KL}(p^*(x);q(x))\right]$ denote the \emph{entropic information loss}.
We can relate the
information loss (based on squared error) to the entropic information loss (based on KL-divergence).
\begin{proposition}
\label{prop:pinsker}
For a calibrated predictor $z:\X \to [0,1]$, where
$\supp(z) \subseteq \{0,1\}\cup [\alpha,1-\alpha]$ for some
constant $\alpha > 0$, the
entropic information loss is within a constant factor
of the information loss.
\begin{equation*}
L_S(p^*;z) \le
2\ln(2) \cdot \LH_S(p^*;z)
\le \frac{1}{\alpha} \cdot L_S(p^*;z)
\end{equation*}
\end{proposition}
Proposition~\ref{prop:pinsker} is a direct corollary of Pinsker's inequality, which relates the KL-divergence to the statistical distance between two probability distributions.
As in Proposition~\ref{prop:infoloss} we can show that this entropic information loss can be expressed as a difference in entropic information content for calibrated predictors.
\begin{proposition}
\label{prop:klloss}
Let $p^*:\X \to [0,1]$ denote the Bayes optimal predictor.  Suppose for $S \subseteq \X$, $z:\X \to [0,1]$ is calibrated on $S$.
Then,
\begin{equation*}
\LH_S(p^*;z) = \IH(p^*) - \IH(z).
\end{equation*}
\end{proposition}
\begin{proof}
Again, we expand the entropic information loss and rearrange.
\begin{align*}
\E_{x \sim S}\left[D_{KL}(p^*(x); z(x))\right] &= \E_{x \sim S}\left[p^*(x)\cdot \log\left(\frac{p^*(x)}{z(x)}\right) + (1-p^*(x)) \cdot \log\left(\frac{p^*(x)}{1-z(x)}\right)\right]\\
&= \E_{x \sim S}\left[p^*(x)\cdot \log\left(\frac{1}{z(x)}\right) + (1-p^*(x)) \cdot \log\left(\frac{1}{1-z(x)}\right)\right] - \E_{x \sim S}\left[H_2(p^*(x))\right]
\end{align*}
Leveraging the fact that $z$ is calibrated, we expand the first term as follows.
\begin{align*}
&\phantom{=}\E_{x \sim S}\left[p^*(x)\cdot \log\left(\frac{1}{z(x)}\right) + (1-p^*(x)) \cdot \log\left(\frac{1}{1-z(x)}\right)\right]\\
&=\sum_{v \in \supp(z)}\R_S^z(v) \cdot \E_{x \sim S_v}\left[p^*(x) \cdot \log\left(\frac{1}{\E_{x' \sim S_v}[p^*(x')]}\right)
+ (1-p^*(x))\cdot \log\left(\frac{1}{1-\E_{x' \sim S_v}[p^*(x')]}\right)
\right]\\
&=\sum_{v \in \supp(z)}\R_S^z(v)\cdot\left(
\E_{x \sim S_v}\left[p^*(x)\right] \cdot \log\left(\frac{1}{\E_{x \sim S_v}[p^*(x)]}\right)
+ (1-\E_{x \sim S_v}\left[p^*(x)\right])\cdot \log\left(\frac{1}{1-\E_{x \sim S_v}[p^*(x)]}\right)
\right)\\
&= \sum_{v \in \supp(z)}\R_S^z(v)\cdot H_2\left(\E_{x \sim S_v}[p^*(x)]\right)\\
&= \sum_{v \in \supp(z)}\R_S^z(v)\cdot H_2(v)\\
&= \E_{x \sim S}\left[H_2(z(x))\right]
\end{align*}
Thus, combining the equlaities, we see that
\begin{equation*}
\LH_S(p^*;z) = \E_{x \sim S}[H_2(z(x))] - \E_{x \sim S}[H_2(p^*(x))] = \IH_S(p^*) - \IH_S(z).
\end{equation*}
\end{proof}
As a final note, we observe that the entropic information content of a calibrated predictor can be related to the expected log-likelihood of the predictor.
Specifically, given a collection of labeled data $(x_1,y_1),\hdots,(x_m,y_m)$ where $y_i \sim \Ber(p^*(x_i))$, the likelihood function $\L(z;\set{(x_i,y_i)})$ is given as follows.
\begin{equation*}
\L(z;\set{(x_i,y_i)})
= \prod_{i=1}^m z(x_i)^{y_i} \cdot (1-z(x_i))^{1-y_i}
\end{equation*}
As such, we say the normalized log-likelihood $\ell(z;\set{(x_i,y_i)})$ is given as
\begin{equation*}
\ell(z;\set{(x_i,y_i)}) = \frac{1}{m}\sum_{i=1}^m y_i \cdot \log(z(x_i)) + (1-y_i)\cdot \log(1-z(x_i)).
\end{equation*}
Suppose the samples $(x_1,y_1),\hdots,(x_m,y_m)$ are drawn such that each $x_i \in S$ for some subset $S \subseteq \X$, then the expected log-likelihood of $z$ is given as follows.
\begin{align*}
\E_{\substack{x_i \sim S\\y_i \sim \Ber(p^*(x))}}\left[\ell(z;\set{(x_i,y_i)})\right]
&= \E_{\substack{x \sim S\\y \sim \Ber(p^*(x))}}\left[y \cdot \log(z(x)) + (1-y)\cdot \log(1-z(x))\right]\\
&= \E_{x \sim S}\left[p^*(x) \cdot \log(z(x)) + (1-p^*(x))\cdot \log(1-z(x))\right]\\
&= - \E_{x \sim S}[H_2(z(x))]\addtag\label{likelihood:final}\\
&= \IH_S(z)-1
\end{align*}
where (\ref{likelihood:final}) follows from the analysis above given in the proof of Proposition~\ref{prop:klloss}.
In other words, as we increase the information content of a calibrated predictor, in expectation, the likelihood of the predictor increases (in expectation over a fresh sample of data).
At the extreme, the calibrated predictor that maximizes the expected likelihood is $p^*$.

\end{document}